\documentclass{article}

\usepackage[margin=1.4in]{geometry}

\usepackage[utf8]{inputenc} % allow utf-8 input
\usepackage[T1]{fontenc}    % use 8-bit T1 fonts
\usepackage{hyperref}       % hyperlinks
\usepackage{url}            % simple URL typesetting
\usepackage{booktabs}       % professional-quality tables
\usepackage{amsfonts}       % blackboard math symbols
\usepackage{nicefrac}       % compact symbols for 1/2, etc.
\usepackage{microtype}      % microtypography

\usepackage{amsmath}
\usepackage{algorithm, algorithmic}
\usepackage{amssymb}
\usepackage{amsthm}
\usepackage{xcolor}

\newtheorem{theorem}{Theorem}[section]

\newtheorem{lemma}[theorem]{Lemma}
\newtheorem{corollary}[theorem]{Corollary}
\newtheorem{definition}[theorem]{Definition}

\newcommand{\bs}[1]{\boldsymbol{#1}}
\newcommand{\spl}[1]{X^{#1}}

\newcommand{\bdis}[1]{\bs{d}^{#1}}
\newcommand{\cst}[2]{{c}^{#1}_{#2}}
\newcommand{\dis}[2]{d^{#1}_{#2}}

\newcommand{\ind}[1]{\mathcal{I}\left(#1\right)}

\newcommand{\na}[1]{[#1]}

\newcommand{\real}[1]{\mathbb{R}^{#1}}

\newcommand{\lr}{\eta}

\newcommand{\pder}[1]{\partial_{#1}}

\newcommand{\der}{\nabla}

\newcommand{\bset}{X^*}

\newcommand{\lf}[1]{\ell^{#1}}

\newcommand{\n}{n}

\newcommand{\bone}[1]{\bs{1}^{#1}}

\newcommand{\bfc}{\bs{c}}
\newcommand{\fc}{c}
\newcommand{\srn}{\sqrt{N}}

\newcommand{\lact}{\mathcal{B}}

\newcommand{\nac}{f}
\newcommand{\lac}{x}
\newcommand{\ntr}{T}
\newcommand{\mdis}{D}
\newcommand{\mcos}{C}

\newcommand{\nsi}{N}
\newcommand{\pows}[1]{\mathcal{P}(#1)}
\newcommand{\bocv}[1]{\bs{c}^{#1}}
\newcommand{\bccm}[1]{\bs{d}^{#1}}
\newcommand{\ocv}[2]{c_{#2}^{#1}}
\newcommand{\ccm}[2]{d_{#2}^{#1}}
\newcommand{\sels}[1]{X^{#1}}
\newcommand{\strat}[1]{\sigma_{#1}}
\newcommand{\spx}[1]{\Delta_{#1}}

\newcommand{\bnac}{\bs{f}}

\newcommand{\bstrat}{\bs{\sigma}}

\newcommand{\bcstr}[1]{\bs{\gamma}^{#1}}
\newcommand{\cstr}[2]{{\gamma}^{#1}_{#2}}
\newcommand{\ints}{\int}

\newcommand{\genr}{R}
\newcommand{\grd}[2]{\nabla#1(#2)}
\newcommand{\inorm}[1]{\|#1\|_{\infty}}
\newcommand{\wtf}[1]{w^{#1}}
\newcommand{\bwtf}[1]{\bs{w}^{#1}}

\newcommand{\bnstrat}[1]{\bs{\sigma}^{\operatorname{#1}}}
\newcommand{\nstrat}[2]{{\sigma}^{\operatorname{#1}}_{#2}}
\newcommand{\pgrd}[3]{\partial_{#3}#1(#2)}

\newcommand{\dpm}[2]{\delta_{}(#2)}
\newcommand{\gens}{\mathcal{S}}
\newcommand{\indi}[1]{\mathcal{I}(#1)}

\newcommand{\avlsn}[2]{\mathcal{L}\left(#1,#2\right)}

\newcommand{\mts}[1]{\mu_{#1}}
\newcommand{\biid}{\bs{s}}
\newcommand{\iid}[1]{s_{#1}}
\newcommand{\trsh}{\theta}

\newcommand{\act}[1]{\mathcal{X}_{#1}}
\newcommand{\lost}[1]{\mathcal{F}_{#1}}
\newcommand{\cplx}[1]{\lambda_{#1}}
\newcommand{\ntri}[1]{T_{#1}}
\newcommand{\g}{\mathcal{G}}
\newcommand{\sfun}[2]{\left\langle#1,#2\right\rangle}
\newcommand{\rplus}{\mathbb{R}^+}

\newcommand{\strs}[1]{\Omega_{#1}}
\newcommand{\gnrs}[2]{R_{#1}^{#2}}
\newcommand{\abs}[2]{\mathcal{Q}(#1,#2)}
\newcommand{\ity}{\omega}
\newcommand{\fl}{\operatorname{FL}}
\newcommand{\ual}{0}
\newcommand{\fll}[2]{\ell_{#1,#2}}
\newcommand{\nosf}{\nu}

\newcommand{\dt}{\operatorname{DT}}
\newcommand{\shift}[2]{#1^{[#2]}}

\newcommand{\ltr}[1]{l_{#1}}
\newcommand{\ttr}[1]{\trsh_{#1}}
\newcommand{\tatr}[1]{\tau_{#1}}
\newcommand{\maxl}[1]{q^{#1}}
\newcommand{\rl}{L}
\newcommand{\rk}{\Gamma}
\newcommand{\h}{\mathcal{H}}
\newcommand{\tran}{\mathcal{W}}
\newcommand{\pt}[1]{\alpha_{#1}}
\newcommand{\ft}[1]{\beta_{#1}}
\newcommand{\at}[1]{\psi_{#1}}
\newcommand{\lm}[1]{\phi_{#1}}
\newcommand{\co}[1]{\operatorname{CO}(#1)}
\newcommand{\fns}{K}
\newcommand{\flf}[1]{\fl_{#1}}
\newcommand{\flb}[1]{{\fl}^{\circ}_{#1}}
\newcommand{\spf}{X}
\newcommand{\tro}{\mathcal{Y}}
\newcommand{\trt}{\mathcal{Z}}
\newcommand{\ord}{v}
\newcommand{\bdv}{\bs{d}}
\newcommand{\dv}[1]{d_{#1}}
\newcommand{\tfns}{\Upsilon}
\newcommand{\hphi}{\hat{\phi}}
\newcommand{\boc}{\bocv{}}
\newcommand{\oc}[1]{\ocv{}{#1}}
\newcommand{\rele}[2]{\mathcal{D}(#1,#2)}
\newcommand{\bwst}{\bs{u}}
\newcommand{\wst}[1]{u_{#1}}
\newcommand{\bgv}[1]{\bs{g}^{#1}}
\newcommand{\gv}[2]{g^{#1}_{#2}}
\newcommand{\nrmc}[1]{Z_{#1}}
\newcommand{\nosg}{\rho}
\newcommand{\sets}{\mathcal{X}}
\newcommand{\on}[1]{\operatorname{#1}}

\newcommand{\initialise}[1]{\bs{\on{initialise}}_{#1}}
\newcommand{\play}[1]{\bs{\on{play}}_{#1}}
\newcommand{\update}[1]{\bs{\on{update}}_{#1}}
\newcommand{\binitialise}[1]{\bs{\on{initialise}}_{#1}}
\newcommand{\bplay}[1]{\bs{\on{play}}_{#1}}
\newcommand{\bupdate}[1]{\bs{\on{update}}_{#1}}
\newcommand{\ona}{\on{A}}
\newcommand{\iput}[1]{y_{#1}}
\newcommand{\oput}[1]{X_{#1}}
\newcommand{\oneg}{\on{CO}}

\newcommand{\gneg}[2]{\tilde{\on{CO}}(#1,#2)}
\newcommand{\onflo}{\on{FL^{\bullet}}}
\newcommand{\onflt}{\on{FL^{\circ}}}
\newcommand{\return}{\bs{\on{return}}}
\newcommand{\la}{\leftarrow}
\newcommand{\f}{f}
\newcommand{\lchild}[1]{{\triangleleft}(#1)}
\newcommand{\rchild}[1]{{\triangleright}(#1)}
\newcommand{\cod}{\tilde{\on{CO}}}
\newcommand{\flo}{\tilde{\on{FL}^{\bullet}}}
\newcommand{\flt}{\tilde{\on{FL}^{\circ}}}
\newcommand{\onfl}{\on{FL}}
\newcommand{\flg}{\tilde{\on{FL}}}
\newcommand{\glob}{\bs{\on{global}}}
\newcommand{\seg}[1]{S_{#1}}
\newcommand{\stf}[1]{m_{#1}}
\newcommand{\zfun}{\zeta}
\newcommand{\des}[1]{{\Downarrow}(#1)}

\allowdisplaybreaks

\title{\bf Online Learning of Facility Locations}

\author{{\bf Stephen Pasteris}\\ University College London\\ London, UK\\ \texttt{s.pasteris@cs.ucl.ac.uk}\\
\and {\bf Ting He}\\ Pennsylvania State University\\ University Park, PA, USA\\ \texttt{t.he@cse.psu.edu}\\
\and {\bf Fabio Vitale}\\ University of Lille \& INRIA\\ Lille, France\\ \texttt{fabio.vitale@inria.fr}\\
\and {\bf Shiqiang Wang}\\ IBM Research T. J. Watson\\ Yorktown Heights, NY, USA\\ \texttt{wangshiq@us.ibm.com}\\
\and {\bf Mark Herbster}\\ University College London\\ London, UK\\ \texttt{m.herbster@cs.ucl.ac.uk}}

\begin{document}

\date{}

\maketitle

\begin{abstract}
In this paper, we provide a rigorous theoretical investigation of an online learning version of the Facility Location problem which is motivated by emerging problems in real-world applications. In our formulation, we are given a set of sites and an online sequence of user requests. At each trial, the learner selects a subset of sites and then incurs a cost for each selected site and an additional cost which is the price of the user's connection to the nearest site in the selected subset. The problem may be solved by an application of the well-known Hedge algorithm. This would, however, require time and space exponential in the number of the given sites, which motivates our design of a novel quasi-linear time algorithm for this problem, with good theoretical guarantees on its performance.
\end{abstract}

\section{Introduction}
In this paper we consider an online learning version of the Facility location problem where users need to be served one at a time  in a sequence of trials. The goal is to select, at each trial, a subset of a given set of sites, and then pay a loss equal to their total ``opening cost'' plus the minimum ``connection cost'' for connecting the user to one of the sites in the subset. More precisely, we are given a set of $N$ sites. At the beginning of each trial, an opening cost and a connection cost for the arriving user are associated with each site and are unknown. At each trial, the learner has to select a subset of sites and incurs a loss given by the {\em minimum} connection cost over the selected sites plus the sum of the opening costs of all selected sites. After each subset selection, the opening and connection costs of all sites are revealed. 

To solve this problem, we design and rigorously analyse an algorithm which belongs to the class of online learning algorithms that make use of the Exponentiated gradient method \cite{Kivinen97}. We measure, and rigorously analyse, the performance of our method by comparing its cumulative loss with that of any fixed subset of sites. Moreover, our algorithm is very scalable: it requires a per-trial time quasi-linear in $N$ and logarithmic in the number of trials, and requires a total space linear in $N$.

The Facility location problem is one of the most well-studied problems in the Operations Research literature \cite{Cornuejols90,Laoutaris07,Shmoys01}. 
In this work we focus on an online version of this problem, which encompasses problems where both the opening and connection costs of the sites change over time.
As far as we are aware, this is the first investigation of this online learning version of the Facility location problem. Our formulation is general and very natural, and can 
model several real-world applications. In the {\em mobile edge computing} context, computing capabilities are pushed from the centralised cloud to the network edge~\cite{Wang18,Wang19A,Wang19B}. The users that need to be served move dynamically and the main challenge consists in reducing the user-perceived latency. In our problem formulation, the connection cost of the sites and can be interpreted as the transmission cost. The opening costs can be viewed as arising from the cost incurred by the resource contention among different service entities. It is natural to assume that this cost grows proportionally to its demand and that it commonly cannot be deduced from information available before having to select the subset of sites. 

Concrete problems like selecting and matchmaking groups of players with low latency to each other in  {\em online multiplayer gaming} can also be cast into this framework.  
This is a very challenging problem because of the real-time interaction required for online computer games, the difficulties in predicting user request locations and the lack of guarantees of timely delivery and network capacity, which in turn can be viewed as related to the connection costs of the sites in our formulation. Another example is represented by  {\em robo-taxis (self-driving taxis) services} which are being piloted
 in a number of major metropolitan areas.  In this example, the connection cost of each site can be viewed again as depending on several unpredictable variables which will be typically revealed after the service is used. Furthermore, the opening cost, i.e., the cost of activating a service, can be viewed, for instance,  as arising by different services competing for the same resources. 

More generally, the connection cost for each site can be viewed as defined by the fixed location of the site (e.g., an edge server~\cite{Wang19B}) and the location of the current service request (e.g., the edge server directly covering the requesting user). Then, each trial corresponds to the service of one request, which is assumed to be delay-sensitive and needs to be served immediately (e.g.,  matchmaking requests for multi-player online games). This interpretation implies that our formulation models a discrete event-driven system, where each trial starts with the placement of the service and ends with the arrival of a new request, not necessarily from the same user. This justifies the assumption of arbitrarily changing connection costs, although the location of a real user will have a temporal correlation. Switching from one service placement to another generally incurs some operation cost and some delay. In this work, we assume that the service is stateless (i.e., no migration needed) so that the operation cost is mainly the cost of activating the service at the newly selected sites. Furthermore, we assume that the inter-arrival time between consecutive requests is relatively large compared to the service switching time, so that the switching delay can be ignored (we leave the consideration of switching cost/delay to future work).
 
We point out that  our problem formulation is not restricted to two-dimensional (geographic) distances, nor even metric spaces.
Our formulation captures opening and connection cost models that are very general. More specifically, the connection costs in our model are not required to be metric-conforming. 

\subsection{Related Work}

Our problem is an online learning version of the classic ``(Uncapacitated) Facility location problem'' (FLP)~\cite{Cornuejols90,Laoutaris07,Shmoys01} in which all costs and all $T$ users are given a-priori and the aim is to select a set of sites that approximately minimises the sum of the ``opening costs'' of those sites plus the sum of the minimum ``connection cost'' from each user to the selected sites. With no other assumptions it has been shown that, by reduction of the ``Weighted Set Cover problem'' \cite{Chvatal79} to FLP, it is impossible (unless P=NP) to get a polynomial-time algorithm that obtains an approximation ratio better than logarithmic in $T$ in general \cite{Dinur14}. 

FLP reduces to the Weighted Set Cover (WSC) problem, in which the greedy algorithm for set cover can give an approximation ratio that is logarithmic in $T$. In the reduction, each subset of users appears $N$ times: each time with a corresponding site. Given a subset of users and a site, the weight of that instance of the subset is the opening cost of the site plus the sum of the connection costs (to that site) of the users in the subset. Although the size (i.e., the number of given subsets) of the equivalent WSC problem is exponential in $T$, the greedy algorithm will, on each iteration, only select a subset from one of $N T$ known subsets (where $N$ is the number of possible sites) and will hence run in polynomial time. 

Algorithms have been developed for online linear optimisation where the set of allowed vectors is in an arbitrary compact subset of $\real{m}$ \cite{Fujita13,Kakade07,Kalai05,Hazan18}. These algorithms utilise an $\alpha$-approximation algorithm for the offline linear optimisation problem. The online learning of a Weighted Set Cover (OWSC) is such a problem and the greedy algorithm is a $(\log(T)+1)$-approximation algorithm. Hence, due to the reduction of FLP to WSC, it would appear that this could solve our problem. Hence, we will now argue that our problem does not reduce to OWSC in the way that FLP reduces to WSC (albeit with a number of sets exponential in $T$).  On each trial we have a single user so the base set of WSC in the reduction contains only the single user. This means that all sets in the reduction cover the base set. Since every set in the cover corresponds to a single site, and the set covers the base set, the weight of that set must be equal to the sum of the opening and connection costs of that site. The sum of the weights of two sets therefore does not necessarily equal the loss incurred by selecting both those sites (in our problem), which is equal to the sum of the opening costs of the sites plus the minimum (not the sum) of their connection costs. Hence, OWSC does not correspond to our problem.

When the distances satisfy the requirement of a metric (which is not enforced in our problem), then constant approximation ratio algorithms for FLP are known \cite{Charikar99b,Guha99,Jain01}. We will now discuss using such algorithms with the well known ``Follow the Leader'' (FTL) strategy. FTL is perhaps the most simple online learning algorithm: the action we choose on any trial is that which would minimise the sum of the losses of the previous trials if it had been selected on all those trials (we call this action ``the leader''). Due to the NP-hardness of FLP we cannot expect to be able to do FTL exactly (with an efficient algorithm) but we could use the greedy algorithm (or constant factor approximation algorithms for metric cases) to approximate the leader, and then use the approximate leader instead. However, doing this results in a deterministic algorithm and we prove, in Appendix \ref{failsec}, that no deterministic algorithm can achieve the (expected) loss bound of our algorithm. FTL hence does not have the performance guarantee of our algorithm, even if the actual leader could be found. Also, FTL with the greedy strategy has a per trial time complexity of $\Theta(\n T)$ whilst that of ours is only $\Theta(\n\ln(\n)\ln(\ntr))$.

An improvement over the FTL approach is that of ``Hedge''~\cite{Freund97} which maintains a weight for each possible action and draws, on any trial, an action with probability proportional to its weight. Actions which have performed well so far have higher weights than those that have not performed well. Unlike FTL, Hedge~has a non-vacuous bound for our problem. However, each subset of sites is an action so there are exponentially many, implying that Hedge~has an exponential time and space complexity. The idea of Hedge~has been extended to algorithms such as ``Component Hedge'' \cite{Koolen10} where, like our problem, each action is a subset of a set of components (in our case the sites). However, Component Hedge assumes that the loss on each trial is a weighted sum of the components in the action so cannot deal with the connection cost (which is a minimisation over sites in the action). Like Hedge, our algorithm is one of a family of algorithms that use the ``Exponentiated Gradient method'' \cite{Kakade08,Kivinen97} to update probability distributions by using gradients. 

A variant of FLP which is close in spirit to ours is the ``Online Facility Location problem'' (OFL)~\cite{Fotakis11,Cygan18,Meyerson01} which has been extensively studied. In this problem, like in ours, the game runs over a set of trials, with a single user request on each trial. In OFL, the costs are fixed and if a site has been selected on any trial we pay its cost only once. Our problem is different in a number of ways. 1) In OFL, the location of the next user is seen before choosing a potentially new site, whilst in our problem the next user location is unknown. 2) In our problem, the opening costs vary from trial to trial, whilst in OFL, they are fixed.  3) OFL assumes the connection costs satisfy the conditions of a metric, whilst ours does not have to. The two problems are sufficiently distinct so that a methodology for one does not imply a methodology for the other.

Perhaps the closest work to ours is that of ``MaxHedge''~ \cite{Pasteris19a}. In the problem that MaxHedge~solves, the learner, like in our problem, picks a subset of sites, each with unknown cost, and then a user appears. The difference from our problem is that, in the problem of MaxHedge, the user gives us a reward based on its distance rather than giving us a penalty (the connection cost) based on distance. The objective is to maximise the profit which is the difference between the reward and the total cost of selecting the sites. Problems involving the maximisation of a profit are very different from those of minimising a loss, in that having an $\alpha$-approximation algorithm for one does not give an approximation algorithm for the other. Hence the problem of MaxHedge~ is very different from ours. The algorithms and analysis of MaxHedge~and our algorithm are also very different, although we utilise the sorting of sites that MaxHedge~does, which was in turn inspired by \cite{Pasteris19b}.

\subsection{Structure of the Paper}
This paper is structured as follows. In Subsection \ref{definitions}, we define the notation that is used in the main body of paper. In Section \ref{PDAR} we introduce our problem and give the loss-bound of our algorithm.  In Section \ref{algsec} we give our algorithm and describe its mechanics. In Appendix \ref{ecs} we give two subrountines in order to make the algorithm of Section \ref{algsec} efficient.  In Appendix \ref{oogdefsec} we define the notation used in the analysis of the algorithm. In Appendix \ref{oogsec} we give the theoretical concepts that underly the analysis of the algorithm. In Appendix \ref{devstratsec} we mathematically formulate and analyse our algorithm. In Appendix \ref{proofsec} we prove all of the theorems that were stated during the analysis of the algorithm (i.e. in appendices \ref{ecs}, \ref{oogsec} and \ref{devstratsec}). Appendix \ref{infclsec} describes how online classification can be formulated in terms of the theory of Appendix \ref{oogsec} and is intended as an example for the reader to familiarise themselves with the theoretical concepts. In Appendix \ref{failsec} we prove that no deterministic algorithm, e.g. follow the (approximate) leader, can achieve the bound on the (expected) loss that our algorithm does.

\subsection{Definitions}\label{definitions}
We now define the notation used in the main body of the paper. We define $\mathbb{R}^+:=\{x\in\mathbb{R}~|~x\geq0\}$. Given real numbers $x, x'\in\mathbb{R}$ we define $[x,x']=\{y\in\mathbb{R}~|~x\leq y \leq x'\}$. We define $\mathbb{N}$ to be the set of positive integers. Given $x\in\mathbb{R}^+$ we define $\lceil x\rceil:=\min\{n\in\mathbb{N}~|~n\geq x\}$. Given $n\in\mathbb{N}$ we define $[n]:=\{m\in\mathbb{N}:m\leq n\}$. Given any vector $\bs{x}\in\mathbb{R}^P$, for some $P\in\mathbb{N}$, we define $x_i$ to be it's $i$-th component.  Given a set $\gens$ we define $\pows{\gens}$ to be the power-set of $\gens$: that is, the set of all subsets of $\gens$. Given $P\in\mathbb{N}$, a subset $\gens$ of $\mathbb{R}^{P}$, a differentiable function $f:\gens\rightarrow\mathbb{R}$, and some $\bs{x}\in\gens$ we define $\grd{f}{\bs{x}}$ to be the gradient of $f$ evaluated at $\bs{x}$. In addition we define $\pgrd{f}{\bs{x}}{i}$ to be the $i$th component of $\grd{f}{\bs{x}}$. Given $P\in\mathbb{N}$ we define $\spx{P}$ to be the set of vectors $\bs{x}$ in $\mathbb{R}^P$ such that $\sum_{i\in \na{P}}x_i=1$ and for all $i\in\na{P}$ we have $x_i\geq 0$. Given a predicate $\pi$, we define $\indi{\pi}$ to be its indicator function: that is, $\indi{\pi}:=0$ if $\pi$ is false and $\indi{\pi}:=1$ if $\pi$ is true. Given $P\in\mathbb{N}$ we define $\bs{1}^P$ to be the vector in $\mathbb{R}^P$ in which each component is equal to $1$.

\section{Problem Description and Result}\label{PDAR}
We now introduce an online learning version of the classic ``Facility location problem'', which we call the ``Facility location game''. The Facility location game is based on the following family of functions. We have constants $\mcos$ and $\mdis$ and define $\sets:=\pows{\na{\nsi}}\setminus\{\emptyset\}$ for some given natural number $\nsi$. Given $\bocv{}\in[0,\mcos]^{\nsi}$ and $\bccm{}\in[0,\mdis]^{\nsi}$ we define the function $\fll{\bocv{}}{\bccm{}}:\sets\rightarrow \rplus$ by:
$$\fll{\bocv{}}{\bccm{}}(\sels{}):=\sum_{i\in\sels{}}\ocv{}{i}+\min_{i\in\sels{}}\ccm{}{i}$$
Intuitively we have $\nsi$ sites and a single user. Each site $i$ has an ``opening cost'' $\ocv{}{i}$, which is the cost of opening a facility there, and a ``connection cost'' $\ccm{}{i}$, which is the cost of connecting the user to it. We open facilities on the set $\sels{}$ of selected sites. We pay the total cost $\sum_{i\in\sels{}}\ocv{}{i}$ for opening the facilities plus the cost $\min_{i\in\sels{}}\ccm{}{i}$ of connecting the user to the nearest open facilility. The Facility location game is a repeated game between Learner and Nature that runs over trials $t=1,2,\ldots \ntr$. On trial $t$:
\begin{enumerate}
\item Nature selects $\bocv{t}\in[0,\mcos]^{\nsi}$ and $\bccm{t}\in[0,\mdis]^{\nsi}$ but does not reveal them to Learner.
\item Learner chooses $\sels{t}\in\sets$.
\item $\bocv{t}$ and $\bccm{t}$ are revealed to Learner.
\item Learner incurs loss $\fll{\bocv{t}}{\bccm{t}}(\sels{t})$
\end{enumerate}
The goal of Learner is to choose $\sels{t}$ in such a way that it incurs a small cumulative loss $\sum_{t=1}^{\ntr}\fll{\bocv{t}}{\bccm{t}}(\sels{t})$ in expectation (over an internal randomisation of its choices). The problem of choosing, in retrospect, the set $\sels{*}$ that minimises the objective function $\sum_{t=1}^{\ntr}\fll{\bocv{t}}{\bccm{t}}(\sels{*})$ is the famous ``Facility location problem''. We seek an efficient algorithm for Learner whose cumulative loss is bounded respect to this this objective function.

In this paper we will present an efficient algorithm for Learner in which, for any set $\sels{*}\in\sets$, we have:
\begin{equation}\label{mbeq}
\mathbb{E}\left(\sum_{t=1}^{\ntr}\fll{\bocv{t}}{\bccm{t}}(\sels{t})\right)\in\mathcal{O}\left(\ln(\ntr)\sum_{t=1}^{\ntr}\fll{\bocv{t}}{\bccm{t}}(\sels{*})+|\sels{*}|(\mcos+\mdis)\ln(\ntr)\sqrt{{\ln(\nsi)}{\ntr}}\right)
\end{equation}
The algorithm is efficient in that it runs in a time of $\mathcal{O}(\nsi\ln(\nsi)\ln(\ntr))$ per trial.

We now argue that this bound on the expected cumulative loss is good for a polynomial-time algorithm. We first consider the first term on the right hand side of Equation \eqref{mbeq}. As noted above, the problem of minimising $\sum_{t=1}^{\ntr}\fll{\bocv{t}}{\bccm{t}}(\sels{*})$ is the facility location problem. This problem is $\on{NP}$-hard and it has been shown, via reduction from the set cover problem, that, unless $\on{P}=\on{NP}$, no polynomial time algorithm can achieve an approximation ratio better than $(1-\epsilon)\ln(T)$ in general, for every $\epsilon\in\mathbb{R}^+$ \cite{Dinur14}. With this negative result in hand we do not expect to see a polynomial time algorithm for the Facility location game whose expected loss is smaller than $\mathcal{O}\left(\ln(\ntr)\sum_{t=1}^{\ntr}\fll{\bocv{t}}{\bccm{t}}(\sels{*})\right)$ in general. We now turn to the second term in the right hand side of Equation \eqref{mbeq}. Since the loss of any action (i.e. selection of set $\sels{}$) is bounded above by $\nsi\mcos+\mdis$, and there are $2^{N}-1$ possible actions, the standard analysis of the (exponential time) Hedge algorithm, leads to a regret bound of $\mathcal{O}\left((\nsi\mcos+\mdis)\sqrt{NT}\right)$. This is close to, and often outperformed by, our term $\mathcal{O}\left(|\sels{*}|(\mcos+\mdis)\ln(\ntr)\sqrt{{\ln(\nsi)}{\ntr}}\right)$.

\section{The Algorithm}\label{algsec}

In this section we give our algorithm for Learner, when playing the facility location game. We will build up the algorithm gradually: starting from the classic exponentiated gradient algorithm \cite{Kivinen97} for online convex optimisation on a simplex, and going via two intermediate algorithms for the Facility location game. Each algorithm builds on the last in that it uses the previous algorithm's methods as subroutines in its own methods. The two intermediate algorithms have a parameter $\fns\in\mathbb{N}$ and have bounds on the cumulative loss, relative to a fixed set of sites, only when the fixed set of sites has cardinality equal to $\fns$ and bounded above by $\fns$, respectively. %Our final algorithm is then obtained by performing a ``doubling trick'' with the second intermediate algorithm. 

In Appendix \ref{devstratsec} we will reformulate all the algorithms in this section formally as ``strategies'' for instances of what we call ``online optimisation games'' and analyse their performance. In order to understand Appendix \ref{devstratsec} it is necessary to first read appendices \ref{oogdefsec} and \ref{oogsec} which contain the required definitions and theoretical concepts respectively. The proofs of all theorems in these appendices are to be found in Appendix \ref{proofsec}.

All algorithms in this section have three methods: $\initialise{}$, $\play{}$ and $\update{}(\cdot)$. The method $\initialise{}$ takes no parameters and has no output, $\play{}$ takes no parameters but returns an output, and $\update{}(\cdot)$ takes a single parameter but has no output.  For an algorithm $\ona$ we will refer to its methods as $\initialise{\ona}$, $\play{\ona}$ and $\update{\ona}(\cdot)$, but will drop, on the subscripts, any parameters associated with $\ona$.

Each algorithm runs over trials $t=1,2,\ldots, \ntr$. On each trial $t$ it outputs some object $\oput{t}$ and then receives some input $\iput{t}$. This process is given in Algorithm \ref{alg1}:

\begin{algorithm}\caption{Algorithm $\ona$}\label{alg1}
\begin{algorithmic}
\STATE $\bullet$ $\initialise{\ona}$
\STATE $\bullet$ For trials $t=1,2,\ldots, \ntr$:
\STATE ~~~~~~~  $\bullet$ $\oput{t}\leftarrow\play{\ona}$
\STATE ~~~~~~~  $\bullet$ $\update{\ona}(\iput{t})$
\end{algorithmic}
\label{alg:mw}
\end{algorithm}

The inputs to $\update{\ona}$ are convex functions when $\ona$ is the Exponentiated gradient algorithm $\oneg$, and a pair of (opening and connection cost) vectors when $\ona$ is one of the algorithms for the Facility location game. For the Facility location game algorithms we define $(\bocv{t},\bccm{t}):=\iput{t}$. The outputs of $\play{\ona}$ are vectors when $\ona$ is the Exponentiated gradient algorithm and sets when $\ona$ is one of the algorithms for the Facility location game.

\subsection{The Exponentiated Gradient Method}\label{algs1}

In Algorithm \ref{alg2} we give the methods of our base algorithm $\oneg(N,G)$, which takes, as inputs, convex functions in $[0,1]^N$ and outputs vectors in $\spx{N}$. The parameter $G$ is an upper bound on the magnitude of any component of the (sub)gradient of any of the input functions, anywhere on $\spx{N}$. The name $\oneg$ stands for ``Convex Optimisation'' and it implements the well studied ``Exponentiated gradient method''. The following property is well known:

If the algorithm $\oneg(N,G)$ is inputted with functions $\iput{1},\iput{2},\cdots,\iput{\ntr}\in [0,1]^N$ that obey the above properties, then the output $\oput{1},\oput{2},\cdots,\oput{\ntr}$ satisfies:
\begin{equation}\label{egboundeq}
\sum_{t\in\na{\ntr}}\iput{t}(\oput{t})-\iput{t}(\oput{}^*)\leq 2G\sqrt{\ln(N)/\ntr}
\end{equation}
for any $\oput{}^*\in\spx{N}$. Note that the objective of the exponentiated gradient method is to minimise $\sum_{t\in\na{\ntr}}\iput{t}(\oput{t})$.

When the method $\update{\oneg}$ is called as a subroutine in another algorithm, the line ``$\glob~\lambda\la f(\bs{w})$'' sets a global variable $\lambda$ equal to $f(\bs{w})$. This will be used in our final algorithm $\onfl$.

\begin{algorithm}\caption{$\oneg(N,G)$}\label{alg2}
\begin{algorithmic}
\STATE $\binitialise{}$:
\STATE ~~~~~~~$\bullet$ $\bs{w}\leftarrow \bone{N}/N$
\STATE ~~~~~~~$\bullet$ $\lr\la\frac{1}{G}\sqrt{\frac{\ln(N)}{T}}$
\STATE
\STATE $\bplay{}$:
\STATE ~~~~~~~ $\bullet$ $\return~ \bs{w}$
\STATE
\STATE $\bupdate{}(f)$:
\STATE ~~~~~~~ $\bullet$ $\glob~\lambda\la f(\bs{w})$
\STATE ~~~~~~~ $\bullet$ $\bs{g}\la \der{f}(\bs{w})$
\STATE ~~~~~~~ $\bullet$ For $i\in\na{N}$: $u_i\la w_i\exp(-\lr g_i)$
\STATE ~~~~~~~ $\bullet$ $Z\la\sum_{i\in\na{N}}u_i$
\STATE ~~~~~~~  $\bullet$ For $i\in\na{N}$: $w_i\la u_i/Z$
\end{algorithmic}
\label{alg:mw}
\end{algorithm}

\subsection{An Algorithm for when the Cardinality of a Comparator Set is Known}\label{algs2}

In Algorithm \ref{alg3} we give the methods of our first algorithm $\onflo(N,C,D,\fns)$ for the facility location game; where $N, C$ and $D$ are defined as in Section \ref{PDAR}. Note that we now also have a parameter $\fns$: we will only compare the performance of the algorithm to that of a fixed sets of sites which has cardinality $\fns$. When the method $\play{}$ is called we choose a vector $\bs{p}\in\spx{N}$ and then form the output $\oput{}$ by drawing $\fns\lceil\ln(\ntr)/2\rceil$ sites with replacement from the probability distribution on $\na{N}$ characterised by $\bs{p}$. Let $\bs{p}^t$ be the value of $\bs{p}$ on trial $t$. We will now describe how and why $\bs{p}^t$ is selected:

Let $\f_{\bocv{},\bccm{}}$ be the function $\f$ created in the method $\play{\onflo}$ when it is inputted with $(\bocv{},\bccm{})$. In Appendix \ref{devstratsec} we shall show that the expected value of $\fll{\bocv{}}{\bccm{}}(\oput{})$ is bounded above $\f_{\bocv{},\bccm{}}(\bs{p})$. In Appendix \ref{devstratsec} we also prove that $\f_{\bocv{},\bccm{}}$ is convex and  that magnitude of any component of its gradient, anywhere on $\spx{N}$ is no more than $(C+D)\fns\lceil\ln(\ntr)/2\rceil$. Since the objective is to minimise the expected cumulative loss $\sum_{t\in\na{\ntr}}\fll{\bocv{t}}{\bccm{t}}(\oput{t})$ where $\oput{t}$ is the output of $\play{\onflo}$ on trial $t$, we will, instead, seek to minimise $\sum_{t\in\na{\ntr}}\f_{\bocv{t},\bccm{t}}(\bs{p}^t)$. This is exactly the goal of the exponentiated gradient method, so we use the exponentiated gradient method with inputs $\{\f_{\bocv{t},\bccm{t}}~|~t\in\na{\ntr}\}$ to produce our sequence $\{\bs{p}^{t}~|~t\in\na{\ntr}\}$.

In Appendix \ref{devstratsec} we bound the value of $\sum_{t\in\na{\ntr}}\f_{\bocv{},\bccm{}}(\bs{p}^*)$, minimised over all $\bs{p}^*\in\spx{N}$, which, by utilising Equation \eqref{egboundeq}, gives us:
\begin{equation}\label{fixeq}
\mathbb{E}\left(\sum_{t\in\na{\ntr}}\fll{\bocv{t}}{\bccm{t}}(\oput{t})\right)\leq\lceil \ln(\ntr)/2\rceil\sum_{t\in\na{\ntr}}\fll{\bocv{t}}{\bccm{t}}(\oput{}^*)+(2\fns(C+D)\lceil \ln(\ntr)/2\rceil+D)\sqrt{\ln(N)\ntr}
\end{equation}
for any selection of sites $\oput{}^*$ with cardinality $\fns$.

Of course, to run the algorithm we must sample $\fns\lceil\ln(\ntr)/2\rceil$ sites from a probability distribution over $\na{N}$ characterised by a vector $\bs{p}\in\spx{N}$ and, during the subrountine $\update{\oneg}(\f)$, compute the value and gradient of ${f}(\bs{w})$. In Appendix \ref{ecs} we show how to perform each of these tasks in a time of $\mathcal{O}(N\ln(N)\ln(T))$.

\begin{algorithm}\caption{$\onflo(N,C,D,\fns)$}\label{alg3}
\begin{algorithmic}
\STATE $\binitialise{}$:
\STATE ~~~~~~~$\bullet$ $\tfns\la\fns\lceil\ln(\ntr)/2\rceil$
\STATE ~~~~~~~$\bullet$ $\ona\la\oneg(N,(C+D)\tfns)$
\STATE ~~~~~~~$\bullet$ $\initialise{\ona}$
\STATE
\STATE $\bplay{}$:
\STATE ~~~~~~~ $\bullet$ $\bs{p}\la\play{\ona}$
\STATE ~~~~~~~ $\bullet$ For all $i\in\na{\tfns}$ sample some $k_i\in\na{N}$ with probability $p_{k_i}$
\STATE ~~~~~~~ $\bullet$ $\spf\la \{j\in\na{N}~|~\exists i\in\na{\tfns}:k_i=j\}$
\STATE ~~~~~~~ $\bullet$ $\return~ \spf$
\STATE
\STATE $\bupdate{}(\bocv{},\bdv)$:
\STATE ~~~~~~~ $\bullet$ Sort $\na{N}$ as $\ord(1),\ord(2),\cdots\ord(N)$ such that $\dv{\ord(i+1)}\leq\dv{\ord(i)}$ for all $i\in\na{N}$
\STATE ~~~~~~~ $\bullet$ Define $\f:[0,1]^N\rightarrow\rplus$ by: $$f(\bs{w}):=\tfns\bocv{}\cdot\bs{w}+\dv{\ord(N)}+\sum_{i\in\na{N-1}}\left(\dv{\ord(i)}-\dv{\ord(i+1)}\right)\left(\sum_{j\in\na{i}}w_{\ord(j)}\right)^{\tfns}$$
\STATE ~~~~~~~ $\bullet$ $\update{\ona}(\f)$
\end{algorithmic}
\label{alg:mw}
\end{algorithm}

\subsection{An Algorithm for when a Bound on the Cardinality of a Comparator Set is Known}\label{algs3}

In Algorithm \ref{alg4} we give the methods of our second algorithm $\onflt(N,C,D,\fns)$ for the facility location game. Instead of being able to compare against just fixed sets of sites with cardinality equal to $\fns$, we can now compare against any fixed sets of sites with cardinality bounded above by $\fns$. To do this we add $N$ ``dummy'' sites, each with zero opening cost, and use the algorithm $\onflo(2N,C,C+D,\fns)$ on this extended collection of sites. When $\play{\onflt}$ is called we simply take the output from $\play{\onflo}$, which is a subset of the $2N$ sites, and remove the dummy sites. Since we can't choose the empty set, if all sites in the output of $\play{\onflo}$ are dummy sites then we will simply choose $\{1\}$ as the output of $\play{\onflt}$, which has a total cost of no more than $C+D$. Because of this we assign a connection cost of $C+D$ to all the dummy sites. We can now compare to any fixed set $\oput{}^*$ of sites in $\na{N}$ with cardinality no greater than $\fns$: if $|\oput{}^*|<\fns$ we simply add $\fns-|\oput{}^*|$ dummy sites to it so the cardinality becomes $\fns$ and we can use the bound of $\onflo(2N,C,C+D,\fns)$. Our bound on the expected cumulative loss is then:
\begin{equation}\label{flbe}
\mathbb{E}\left(\sum_{t\in\na{\ntr}}\fll{\bocv{t}}{\bccm{t}}(\oput{t})\right)\leq\left\lceil \frac{\ln(\ntr)}{2}\right\rceil\sum_{t\in\na{\ntr}}\fll{\bocv{t}}{\bccm{t}}(\oput{}^*)+(2\fns(2C+D)\left\lceil \frac{\ln(\ntr)}{2}\right\rceil+(C+D))\sqrt{\ln(2N)\ntr}
\end{equation}
For any subset of sites $\oput{}^*$ with cardinality no greater than $\fns$.

\begin{algorithm}\caption{$\onflt(N,C,D,\fns)$}\label{alg4}
\begin{algorithmic}
\STATE $\binitialise{}$:
\STATE ~~~~~~~$\bullet$ $\ona\la\onflo(2N,C,C+D,\fns)$
\STATE ~~~~~~~$\bullet$ $\initialise{\ona}$
\STATE
\STATE $\bplay{}$:
\STATE ~~~~~~~ $\bullet$ $\hat{\spf}\la\play{\ona}$
\STATE ~~~~~~~ $\bullet$ ${\spf}\la\hat{\spf}\cap\na{N}$
\STATE ~~~~~~~ $\bullet$ If ${\spf}=\emptyset$ then set ${\spf}\la\{1\}$
\STATE ~~~~~~~ $\bullet$ $\return~ {\spf}$
\STATE
\STATE $\bupdate{}(\bocv{},\bdv)$:
\STATE ~~~~~~~ $\bullet$ For all $i\in\na{N}$ set $\hat{c}_i\la \ocv{}{i}$ and $\hat{d}_i\la\dv{i}$
\STATE ~~~~~~~ $\bullet$ For all $i\in\na{2N}\setminus\na{N}$ set $\hat{c}_i\la 0$ and $\hat{d}_i\la C+D$
\STATE ~~~~~~~ $\bullet$ $\update{\ona}(\hat{\bs{c}},\hat{\bs{d}})$
\end{algorithmic}
\label{alg:mw}
\end{algorithm}

\subsection{The Main Algorithm}\label{algs4}

Finally, in Algorithm \ref{alg5} we give the methods of our main algorithm $\onfl(N,C,D)$ which works by performing a ``doubling trick''  with $\onflt$. During the method $\update{\onfl}$, the method $\update{\onflt}$ is called and hence so is $\update{\oneg}$. During the method $\update{\oneg}$ a global variable $\lambda$ is modified. Let $\lambda_t$ be the value of $\lambda$ at the end of trial $t$. We define $a:=\lceil\ln(\ntr/2)\rceil(4C+2D)$ and $b:=C+D$. The trials are divided into segments $\seg{0}, \seg{1}, \seg{2} \cdots $. At the start of segment $\seg{i}$ the algorithm $\onflt(N,C,D,\lceil2^i(a+b)-b)/a\rceil)$ is initialised and runs until the sum of the values $\lambda_t$ over trials $t$ in the segment so far exceeds $2^{i+1}(a+b)\trsh\sqrt{\ln(2N)\ntr}$. When this happens, $\seg{i}$ finishes and $\seg{i+1}$ starts. In appendices \ref{oogsec} and \ref{devstratsec} we analyse our doubling trick (Appendix \ref{GDTss} defines an analyses the doubling trick in general and then Appendix \ref{flgs} applies it to the facility location game) which, combined with Equation \eqref{flbe}, gives a bound, for $\onfl(N,C,D)$, of:
\begin{equation}\label{eqfinmain}
\mathbb{E}\left(\sum_{t=1}^{\ntr}\fll{\bocv{t}}{\bccm{t}}(\sels{t})\right)\in\mathcal{O}\left(\ln(\ntr)\sum_{t=1}^{\ntr}\fll{\bocv{t}}{\bccm{t}}(\sels{*})+|\sels{*}|(\mcos+\mdis)\ln(\ntr)\sqrt{{\ln(\nsi)}{\ntr}}\right)
\end{equation}
for any non-empty set of sites $\sels{*}$. By using the subroutines of Appendix \ref{ecs} the time complexity of this algorithm is only $\mathcal{O}(N\ln(N)\ln(\ntr))$ per trial.

\begin{algorithm}\caption{$\onfl(N,C,D)$}\label{alg5}
\begin{algorithmic}
\STATE $\binitialise{}$:
\STATE ~~~~~~~$\bullet$ $a\la\lceil\ln(\ntr)/2\rceil(4C+2D)$
\STATE ~~~~~~~$\bullet$ $b\la C+D$
\STATE ~~~~~~~$\bullet$ $\trsh\la 1$
\STATE ~~~~~~~$\bullet$ $\fns\la\lceil(\trsh(a+b)-b)/a\rceil$
\STATE ~~~~~~~$\bullet$ $\ona\la\onflt(N,C,D,\fns)$
\STATE ~~~~~~~$\bullet$ $\initialise{\ona}$
\STATE ~~~~~~~$\bullet$ $l\la 0$
\STATE
\STATE $\bplay{}$:
\STATE ~~~~~~~ $\bullet$ $\spf\la\play{\ona}$
\STATE ~~~~~~~ $\bullet$ $\return~ {\spf}$
\STATE
\STATE $\bupdate{}(\bocv{},\bdv)$:
\STATE ~~~~~~~ $\bullet$ $\update{\ona}(\bocv{},\bdv)$
\STATE ~~~~~~~ $l\la l+\lambda$
\STATE ~~~~~~~ $\bullet$ $\on{if}~l\geq 2(a+b)\trsh\sqrt{\ln(2N)\ntr}$:
\STATE ~~~~~~~~~~~~~~~~~$\bullet$ $\trsh\la 2\trsh$
\STATE ~~~~~~~~~~~~~~~~~$\bullet$ $\fns\la\lceil\trsh(a+b)-b)/a\rceil$
\STATE ~~~~~~~~~~~~~~~~~$\bullet$ $\ona\la\onflt(N,C,D,\fns)$
\STATE ~~~~~~~~~~~~~~~~~$\bullet$ $\initialise{\ona}$
\STATE ~~~~~~~~~~~~~~~~~$\bullet$ $l\la 0$
\end{algorithmic}
\label{alg:mw}
\end{algorithm}

\section{Conclusions and Ongoing Work}

In this paper, we have proposed a novel online learning version of the classic Facility Location problem. We have proposed an algorithm for this problem and derived bounds on its expected loss relative to that of the any fixed set of sites. Ongoing work for this problem includes:

\begin{itemize}

\item Complement our study by carrying out an experimental evaluation of our algorithm on real-world datasets. 

\item Extend the algorithm to the bandit setting \cite{Auer02,Robbins52}, when only the total cost, or perhaps total opening cost and minimum connection cost, is revealed on each trial.

\item In some real-world domains, we may have to pay a migration cost to move sites while selecting a subset of sites, or may be limited to how far we can move them. We would like to design an algorithm to handle such problems.

\end{itemize}

\section{Acknowledgements}

This research was sponsored by the U.S. Army Research Laboratory and the U.K. Ministry of Defence under Agreement Number W911NF-16-3-0001. The views and conclusions contained in this document are those of the authors and should not be interpreted as representing the official policies, either expressed or implied, of the U.S. Army Research Laboratory, the U.S. Government, the U.K. Ministry of Defence or the U.K. Government. The U.S. and U.K. Governments are authorized to reproduce and distribute reprints for Government purposes notwithstanding any copyright notation hereon.

\begin{appendix}

\section{Efficient Computation}\label{ecs}

In this section we give two subroutines for the algorithm of Subsection \ref{algs2}, bringing the time complexity of the algorithm of Subsection \ref{algs4} down to $\mathcal{O}(N\ln(N)\ln(\ntr))$ per trial. The proofs of both theorems in this section are to be found in Section \ref{proofsec}

\subsection{Computing $\lambda$ and $\bs{g}$}
When the method $\update{\oneg}$ is run as a subroutine of $\update{\onflo}$ we have some $f$ defined by:

\begin{equation}\label{fdefeq}
f(\bs{p}):=\tfns\bocv{}\cdot\bs{p}+\dv{\ord(N)}+\sum_{i\in\na{N-1}}\left(\dv{\ord(i)}-\dv{\ord(i+1)}\right)\left(\sum_{j\in\na{i}}p_{\ord(j)}\right)^{\tfns}
\end{equation}
and for some $\bs{w}\in\spx{N}$ we need to compute $\lambda:=f(\bs{w})$ and $\bs{g}:=\der{f}(\bs{w})$. Algorithm \ref{algef} shows how to compute both of these quantities in time $\mathcal{O}(N)$. The following theorem asserts the correctness of Algorithm \ref{algef}:

\begin{theorem}\label{fderth}
Given the function $f$ defined in Equation \eqref{fdefeq} and the outputs, $\lambda,\bs{g}$ of Algorithm \ref{algef} we have that $\lambda=f(\bs{w})$ and $\bs{g}=\der{f}(\bs{w})$
\end{theorem}

\begin{algorithm}\caption{Computing $\lambda$ and $\bs{g}$}\label{algef}
\begin{algorithmic}
\STATE $\bullet$ $s_1\la w_{\ord(1)}$
\STATE $\bullet$ For $i=1,2,\ldots N-2$:
\STATE ~~~~~$\bullet$ $s_{i+1}\la s_i+w_{\ord(i+1)}$
\STATE $\bullet$ $\lambda\la \tfns\bocv{}\cdot\bs{w}+\dv{\ord(N)}+\sum_{i\in\na{N-1}}\left(\dv{\ord(i)}-\dv{\ord(i+1)}\right)s_i^{\tfns}$
\STATE $\bullet$ $s'_{N-1}\la \left(\dv{\ord(N-1)}-\dv{\ord(N)}\right)s_{N-1}^{\tfns-1}$
\STATE $\bullet$ For $i=N-1,N-2,\ldots 2$:
\STATE ~~~~~$\bullet$ $s'_{i-1}=s'_{i}+\left(\dv{\ord(i-1)}-\dv{\ord(i)}\right)s_{i-1}^{\tfns-1}$
\STATE ~~~~~$\bullet$ For $i\in\na{N}$ set $g_{\ord(i)}\la \tfns\ocv{}{i}+\tfns s'_i$
\STATE $\bullet$ $\return~\lambda,~ \bs{g}$
\end{algorithmic}
\label{alg:mw}
\end{algorithm}

\subsection{Multiple Samples from a Finite Set}

In Algorithm \ref{algsam} present an algorithm for the efficient sampling of many sites in $\na{N}$ from a probability distribution characterised by a vector $\bs{p}\in\spx{N}$. This algorithm is required in the method $\play{\onflo}$. Algorithm \ref{algsam} has the following notation: given an oriented full binary tree and some internal node $j$ we define $\lchild{j}$ and $\rchild{j}$ to be the left and right child of $j$ respectively. The algorithm has two methods: the method $\bs{\on{initialsation}}(\bs{p})$ constructs the data-structure, taking a time of $\mathcal{O}(N)$. The method $\bs{\on{sample}}$ samples a single point and takes a time of $\mathcal{O}(\ln(N))$. The computational complexities of both methods are clear, whilst the correctness is confirmed by the following theorem:

\begin{theorem}\label{msfsth}
Suppose we have some $i\in\na{N}$ and $\bs{p}\in\spx{N}$. Then, given $\bs{\on{initialsation}}(\bs{p})$ is run a-priori, the method $\bs{\on{sample}}$ returns $i$ with probability $p_i$.
\end{theorem}

\begin{algorithm}\caption{Sampling from a Finite Set}\label{algsam}
\begin{algorithmic}

\STATE $\bs{\on{initialsation}}(\bs{p})$:
\STATE ~~~~~  $\bullet$ $H\la \lceil \ln(N)\rceil$
\STATE ~~~~~  $\bullet$ $N'\la\exp(H)$
\STATE ~~~~~  $\bullet$ For all $i\in \na{N'}\setminus\na{N}$ set $p_i\la 0$
\STATE ~~~~~  $\bullet$ Construct a full, oriented and balanced binary tree $\mathcal{B}$ of height $H$.
\STATE ~~~~~  $\bullet$ Construct an arbitrary bijection $\tau$ from the leaves of $\mathcal{B}$ into $N'$.
\STATE ~~~~~  $\bullet$ For all leaves $j$, of $\mathcal{B}$, set $p'_j\la p_{\tau(i)}$
\STATE ~~~~~  $\bullet$ For $\delta=H-1, H-2, \cdots 1$:
\STATE ~~~~~  ~~~~~~~ $\bullet$ For all nodes $j$ of $\mathcal{B}$ at depth $\delta$ set $p'_j=p'_{\lchild{j}}+p'_{\rchild{j}}$
\STATE
\STATE $\bs{\on{sample}}$:
\STATE ~~~~~  $\bullet$ Set $v_0$ to be the root of of the $\mathcal{B}$.
\STATE ~~~~~ For $\delta=0,1,\cdots H-1$
\STATE ~~~~~  ~~~~~~~ $\bullet$ Sample a random number $r_{\delta}$ uniformly at random from $[0,1]$
\STATE ~~~~~  ~~~~~~~ $\bullet$ If $r_{\delta}\leq p'_{\lchild{v_{\delta}}}/\left(p'_{\lchild{v_{\delta}}}+p'_{\rchild{v_{\delta}}}\right)$ then set $v_{\delta+1}\la\lchild{v_{\delta}}$. Else set $v_{\delta+1}\la\rchild{v_{\delta}}$
\STATE ~~~~~ $\bullet$ $\return~\tau(v_{H})$
\end{algorithmic}
\label{alg:mw}
\end{algorithm}

\section{Definitions}\label{oogdefsec}
We now define the notation used in the analysis of the algorithm.

We let $\ity\in\rplus$ be a surrogate for $\infty$. Throughout the paper we will always assume the limit $\ity\rightarrow\infty$.

Given sets $\gens$ and $\gens'$ we define $\sfun{\gens}{\gens'}$ to be the set of functions from $\gens$ into $\gens'$. 

Given a set $\gens$ and a natural number $T'$ we define $\gens^{T'}$ to be the set of sequences of elements of $\gens$ of length $T'$. Given $\bnac\in\gens^{T'}$ we define $\nac_t$ to be the $t$-th element of the sequence $\bnac$. Given a sequence $\bnac\in\gens^{\ntr'}$ and a function $\ft{}:\gens\rightarrow\hat{\gens}$ for sets $\gens$ and $\hat{\gens}$, we define $\ft{}(\bnac)$ as the sequence in $\hat{\gens}^{\ntr'}$ with $[\ft{}(\bnac)]_t:=\ft{}(\nac_t)$ for all $t\in\na{\ntr'}$.

We define the maximum of the empty-set, $\max\emptyset$, equal to $0$.

\subsection{Measures and Integrals}
We note that, although the definitions in this subsection are about measures, the reader need not be proficient in measure theory to understand the paper.

 When we talk of a ``set'' in what follows, we implicitly assume that the set has a natural associated set of measurable subsets.

 Given a measure $\mu$ on a set $\gens$ and a function $f:\gens\rightarrow\mathbb{R}$ we let $\int_{\gens}f~[\mu]$ be the Lebesgue integral of $f$ with respect to measure $\mu$. Note that if $\gens$ is a finite set then $\int_{\gens}f~[\mu]=\sum_{x\in\gens}f(x)\mu(\{x\})$\,.

A measure $\mu$ on a set $\gens$ is a ``probability measure'' if and only if $\int_{\gens}f~[\mu]=1$, when $f:\gens\rightarrow\mathbb{R}$ is such that $f(x)=1$ for all $x\in\gens$. We let $\spx{\gens}$ be the set of all probability measures on $\gens$.

Given a set $\gens$ and some $x\in\gens$ we define $\dpm{\gens}{x}\in\spx{\gens}$ such that for all measurable subsets $\gens'$ of $S$ we have $[\dpm{\gens}{x}](\gens')=\indi{x\in\gens'}$. Informally, $\dpm{\gens}{x}$ is the probability measure in which all the probability mass in concentrated on $x$, so that any sample from $\dpm{\gens}{x}$ is equal to $x$ (with probability $1$). For all $f:\gens\rightarrow\mathbb{R}$ we have $\int_{\gens}f~[\dpm{\gens}{x}]=f(x)$\,.

Given a measure $\mu$ on a set $\gens$ and a value $a\in\rplus$ we define $a\mu$ to be the measure on $\gens$ defined by $[a\mu](\gens')=a\mu(\gens')$ for all measurable subsets $\gens'$ of $\gens$. Given, in addition, a measure $\mu'$ on $\gens$ we define $\mu+\mu'$ to be the measure on $\gens$ such that $[\mu+\mu'](\gens')=\mu(\gens')+\mu'(\gens')$ for all measurable subsets $\gens'$ of $\gens$.

Given sets $\gens$ and $\hat{\gens}$, a probability measure $\mu\in\spx{\gens}$ and a function $f:\gens\rightarrow\spx{\hat{\gens}}$ we define $\int_{\gens} f~[\mu]$ to be the probability measure $p$ on $\hat{\gens}$ defined by $p(\hat{\gens}')=\int_{\gens}[f(\cdot)](\hat{\gens}')~[\mu]$ where $[f(\cdot)](\hat{\gens}')$ is the function that maps $x\in\gens$ to $[f(x)](\hat{\gens}')$\,.

\section{Online Optimisation Games and the Conversion of Strategies}\label{oogsec}

Here we introduce the theoretical definitions and results that underpin the development of the algorithm. First, we define the notion of an ``online optimisation game'' (OOG) of which many problems in the subject of online learning are instances of. As we define an online optimisation game we also define the notion of a ``strategy'' for Learner and its ``generalised regret'' which measures its performance. After defining OOGs we define two ways in which to convert a class of strategies for one class of OOG into a strategy for another: specifically via ``transformations'' and our ``doubling trick''. The proofs of both theorems in this section are to be found in Section \ref{proofsec}

\subsection{Online Optimisation Games}

We now define an ``Online optimisation game'' (OOG). An OOG $\g$ is defined by the following:

\begin{itemize}
\item $\act{\g}$ is the set of Learner's possible actions.
\item $\lost{\g}$ is a set of ``loss'' functions from $\act{\g}$ into $\rplus$.
\item $\cplx{\g}$ is a ``complexity'' function from $\act{\g}$ into $\rplus$. Actions that have higher complexity are in some sense less natural.
\item $\ntri{\g}$ is the number of trials in the game. We will assume that all OOGs $\g$ in this paper have $\ntri{\g}:=\ntr$.
\end{itemize}

Informally, learning proceeds in trials $t=1,2,...\ntri{}$. On trial $t$:

\begin{enumerate}
\item Nature chooses a loss function $\nac_t\in\lost{\g}$ but does not reveal it to Learner
\item Learner (randomly) chooses an action $\lac_t\in\act{\g}$
\item $\nac_t$ is revealed to Learner
\item Learner suffers loss $\nac_t(\lac_t)$
\end{enumerate}

Given an online optimisation game $\g$ we make the following definitions. Note that we have dropped the subscript $\g$ on its elements.

\begin{definition}
A ``strategy'' is any $\bstrat\in\sfun{\lost{}^{\ntri{}}}{\spx{\act{}}}^{\ntri{}}$ in which, given $t\in\na{\ntri{}}$ and $\bnac,\bnac'\in\lost{}^{\ntri{}}$ with $\nac_{s}=\nac'_{s}$~ for all $s\in\na{t-1}$, we have $\strat{t}(\bnac)=\strat{t}(\bnac')$. Let $\strs{\g}$ be the set of all strategies.
\end{definition}

Informally, a strategy $\bstrat$ defines, on every trial $t$, a probability measure $\strat{t}(\bnac)$ from which $\lac_t$ is drawn. This probability depends on all of Nature's actions $\nac_{t'}$ for all $t'< t$ (since it cant depend of Nature's future selections). Hence, we have the condition that if $\nac_{s}=\nac'_{s}$ for all  $s\in\na{t-1}$, we have $\strat{t}(\bnac)=\strat{t}(\bnac')$.

The expected average loss of a strategy $\bstrat\in\strs{\g}$ when Nature's sequence of selections is $\bnac$ is then:
$$\avlsn{\bstrat}{\bnac}:=\frac{1}{\ntri{}}\sum_{t\in\na{\ntri{}}}\ints_{\act{}}\nac_t~[\strat{t}(\bnac)]$$
To evaluate the performance of a strategy $\bstrat\in\strs{\g}$ we compare its expected average loss to that of a strategy $\bcstr{\lac'}$ that always chooses $\lac_t=\lac'$ for some $\lac'\in\act{}$. Specifically, we define a constant strategy:
\begin{definition}
Given $\lac\in\act{}$ we define $\bcstr{\lac}\in\strs{\g}$ by $\cstr{\lac}{t}(\bnac):=\dpm{\lact}{\lac}$
for all $t\in\na{\ntr}$ and $\bnac\in\lost{}^{\ntri{}}$.
\end{definition}
and we define the ``generalised regret'' $\gnrs{\g}{\bstrat}:\rplus\times\rplus\rightarrow\rplus$ by:
\begin{definition}
Given $\bstrat\in\strs{\g}$ we define its ``generalised regret'' $\gnrs{\g}{\bstrat}\in\sfun{\rplus\times\rplus}{\rplus}$ by:
$$\gnrs{\g}{\bstrat}(\rl,\rk):=\max\{\avlsn{\bstrat}{\bnac}~|~\bnac\in\abs{\rl}{\rk}\}$$
where $\abs{\rl}{\rk}$ is the set of all $\bnac\in\lost{}^{\ntri{}}$ such that there exists $\lac\in\act{}$ with $\cplx{}(\lac)\leq\rk$ and $\avlsn{\bcstr{\lac}}{\bnac}\leq\rl$.
\end{definition}
where unambiguous we will drop the subscript and superscript from $\gnrs{\g}{\bstrat}$.

\subsection{Transformations}
In this paper we will, on two occasions, transform an OOG $\g$ into a special case of an OOG $\h$ (note that this does not mean that $\g$ is itself a special case of $\h$). A ``transformation'' $\tran$ from $\g$ into $\h$ is defined by the following:
\begin{itemize}
\item A function $\at{\tran}:\act{\g}\rightarrow\act{\h}$ with $\cplx{\h}(\at{\tran}(\lac))\leq\cplx{\g}(\lac)$ for all $\lac\in\act{\g}$\,.
\item A function $\pt{\tran}:\act{\h}\rightarrow\spx{\act{\g}}$\,.
\item A function $\ft{\tran}:\lost{\g}\rightarrow\lost{\h}$ such that for all $f\in\lost{\g}$ and $x\in\act{\h}$ we have:
$$[\ft{\tran}(f)](x)\geq\int_{\act{\g}}f~[\pt{\tran}(x)]\,.$$
\end{itemize}

Now suppose we have a transformation $\tran$ from $\g$ into $\h$. We define a function $\lm{\tran}{}:\rplus\times\rplus\rightarrow\rplus$ by:
$$\lm{\tran}(\rl,\rk):=\max\{[\ft{\tran}(f)](\at{\tran}(x))~|~(x,f)\in\act{\g}\times\lost{\g},~f(x)\leq\rl,~\cplx{\g}(x)\leq\rk\}$$
and we define the function $\stf{\tran}:\strs{\h}\rightarrow\strs{\g}$ by:
$$[\stf{\tran}(\bstrat)]_t(\bnac):=\int_{\act{\h}}\pt{\tran}~[\strat{t}(\ft{\tran}(\bnac))]$$
for all $\bstrat\in\strs{\h}$ and $\bnac\in\lost{\g}^{\ntr}$. Note that to sample $x$ from $[\stf{\tran}(\bstrat)]_t(\bnac)$ one simply samples $\hat{x}$ from $\strat{t}(\ft{\tran}(\bnac))$ and then samples $x$ from $\pt{\tran}(\hat{x})$.  The following theorem bounds the generalised regret of $\stf{\tran}$:

\begin{theorem}\label{t1r1}
Suppose we have OOGs $\g$ and $\h$ and a strategy $\bstrat$ for $\h$. Suppose also that we have a transformation, $\tran$, from $\g$ into $\h$ such that $\lm{\tran}$ is bounded above (pointwise) by a function $\hphi$ that is concave in its first argument. We then have:
$$\gnrs{\g}{\stf{\tran}(\bstrat)}(\rl,\rk)\leq\gnrs{\h}{\bstrat}(\hphi(\rl,\rk),\rk)\,.$$
\end{theorem}

\subsection{A General Doubling Trick}\label{GDTss}

In this section we introduce a generalisation of doubling trick which was introduced in \cite{Herbster16}. However, our analysis is sharper, giving us significantly smaller loss bounds.

This subsection deals with complexity functions that evaluate to infinity (i.e. equal to $\ity$) on some actions. To get some intuition behind infinite complexities we advise the reader to first read Section \ref{infclsec}\,.

In this section we consider a general OOG $\g$ with $\min\{\cplx{\g}(\lac)~|~\lac\in\act{\g}\}=1$ and such that there exists a function $\zfun\in\lost{\g}$ in which $\zfun(\lac):=0$ for all $\lac\in\act{\g}$. As stated in the definitions, we use $\ity$ as a surrogate for infinity, taking the limit $\ity\rightarrow\infty$. We will also drop the subscript $\g$ from $\act{\g}$ and $\lost{\g}$.

First, given $\trsh\in\mathbb{R}^+$ we denote by $\cplx{\g}^{\trsh}$ the function from $\act{\g}$ into $\mathbb{R}^+$ defined by:
$$\cplx{\g}^{\trsh}(\lac):=\indi{\cplx{\g}(\lac)>\trsh}\ity$$
and we define the OOG $\g^{\trsh}$ by:
\begin{itemize}
\item $\act{\g^{\trsh}}:=\act{}$\,.
\item $\lost{\g^{\trsh}}:=\lost{}$\,.
\item $\cplx{\g^{\trsh}}:=\cplx{\g}^{\trsh}$\,.
\end{itemize}

Now suppose we have some $\nosf, \nosg\in\rplus$ and, for all $\trsh\in\mathbb{R}^+$, a strategy $\bstrat^{\trsh}$ for the OOG $\g^{\trsh}$ which has a generalised regret of:
$$\gnrs{\g^{\trsh}}{\bstrat^{\trsh}}(\rl,\rk)\leq\nosf\rl+\nosg\trsh+\rk\,.$$

We will now construct a strategy $\bstrat^{\dt}$ for the OOG $\g$ and will bound its generalised regret. We start with the following definitions:

\begin{definition}
Given $\trsh\in\mathbb{R}^+$ we define:
$$\maxl{\trsh}:=\max\left\{\ints_{\act{}}\nac_{t}[\strat{t}^{\trsh}(\bnac)]~\biggr|~\bnac\in\lost{}^{\ntri{}},t\in\na{\ntri{}}\right\}\,. $$
\end{definition}

\begin{definition}
Given $\bnac\in\lost{}^{\ntri{}}$ and $\tau,\tau'\in\na{\ntri{}}$ with $\tau'>\tau$ we define $\shift{\bnac}{\tau,\tau'}\in\lost{}^{\ntri{}}$ by:
\begin{itemize}
\item $\shift{\nac_{t}}{\tau,\tau'}:=\nac_{t+\tau}$~ for all $t\in\na{\tau'-\tau}$\,.
\item $\shift{\nac_{t}}{\tau,\tau'}:=\zfun$~ for all $t\in\na{\ntri{}}$ with $t>\tau'-\tau$.
\end{itemize}
\end{definition}

We consider a fixed choice $\bnac\in\lost{}^{\ntri{}}$, of Nature's selections. We now define the strategy $\bstrat^{\dt}$:

\begin{definition}\label{dtdef}
We define the strategy $\bstrat^{\dt}$, as well as sequences $\bs{\trsh},\bs{l}\in\mathbb{R}^{\ntri{}}$ and $\bs{\tau}\in\mathbb{N}^{\ntri{}}$, iteratively as follows:
\begin{itemize}
\item $\ttr{1}:=1$\,.
\item $\tatr{1}:=0$\,.
\item $\strat{1}^{\dt}(\bnac):=\strat{1}^{\ttr{1}}(\shift{\bnac}{{0, \ntr}})$\,.
\item $\ltr{1}:=\ints_{\act{}}\nac_{1}~[\strat{1}^{\dt}(\bnac)]$\,.
\end{itemize}
For all $t\in\na{\ntri{}-1}$ we define the following:
\begin{itemize}
\item If $\ltr{t}< 2\ttr{t}\nosg$ then:
\begin{itemize}
\item $\ttr{t+1}:=\ttr{t}$\,.
\item $\tatr{t+1}:=\tatr{t}$\,.
\item $\strat{t+1}^{\dt}(\bnac):=\strat{t+1-\tatr{t}}^{\ttr{t}}(\shift{\bnac}{{\tatr{t},\ntr}})$\,.
\item $\ltr{t+1}:=\ltr{t}+\ints_{\act{}}\nac_{t+1}~[\strat{t+1}^{\dt}(\bnac)]$\,.
 \end{itemize}
 \item If $\ltr{t}\geq 2\ttr{t}\nosg$ then:
\begin{itemize}
\item $\ttr{t+1}:=2\ttr{t}$\,.
\item $\tatr{t+1}:=t$\,.
\item $\strat{t+1}^{\dt}(\bnac):=\strat{1}^{\ttr{t+1}}(\shift{\bnac}{{t,\ntr}})$\,.
\item $\ltr{t+1}:=\ints_{\act{}}\nac_{t+1}~[\strat{t+1}^{\dt}(\bnac)]$\,.
 \end{itemize}
\end{itemize}
\end{definition}
The next theorem gives bounds the general regret of the strategy $\bstrat^{\dt}$.
\begin{theorem}\label{dtth}
$\bstrat^{\dt}$ has a generalised regret bounded by:
$$\gnrs{\g}{\bstrat^{\dt}}(\rl,\rk)\leq 5\nosf\rl+8\nosg\rk+\frac{1}{\ntr}\sum_{i=1}^{\lceil\log_2(\rk)\rceil}\maxl{2^i}\,.$$
\end{theorem}

\section{The Development of the Strategy}\label{devstratsec}

In this section we develop a strategy $\bstrat^{\onfl(N,C,D)}$ for the OOG $\flg(N,C,D)$ defined by:

\begin{itemize}
\item $\act{\flg}:=\pows{\na{\nsi}}\setminus\{\emptyset\}$\,.
\item $\cplx{\flg}(\sels{})=|\sels{}|$ for all $\sels{}\in\act{\flg}$\,.
\item $\lost{\flg}:=\left\{\fll{\bocv{}}{\bccm{}}~ |~ \bocv{}\in[0,\mcos]^{\nsi}~\operatorname{and}~ \bccm{}\in[0,\mdis]^{\nsi}\right\}$\,.
\end{itemize}

where, for $\bocv{}\in[0,\mcos]^{\nsi}$, $\bccm{}\in[0,\mdis]^{\nsi}$ and $\sels{}\in\act{\flg}$ we have:
$$\fll{\bocv{}}{\bccm{}}(\sels{})=\sum_{i\in\sels{}}\ocv{}{i}+\min_{i\in\sels{}}\ccm{}{i}\,.$$

Our strategy has a generalised regret $\gnrs{}{}$ bounded as:

$$\gnrs{}{}(\rl,\rk)\in\mathcal{O}\left(\rl\ln(\ntr) +\rk(\mcos+\mdis)\sqrt{\frac{\ln(\nsi)}{T}} \right)\,.$$

To construct our strategy we will move between different OOGs: using the strategy of one OOG to build, via transformations or the doubling trick, a strategy for the next. The sequence of OOGs is as follows:
\begin{enumerate}
\item $\cod$. This is the classic game of online convex optimisation over a simplex.
\item $\flo$. This game is the same as $\flg$ except that it has a parameter $\fns$ such the complexity of a set $\spf$ is $\ity\indi{|\spf|\neq\fns}$.
\item $\flt$. This game is the same as $\flg$ except that it has a parameter $\fns$ such the complexity of a set $\spf$ is $\ity\indi{|\spf|>\fns}$.
\item $\flg$.
\end{enumerate}
For each algorithm $\ona$, described in Section \ref{algsec}, we will, in this section, define and analyse a strategy $\bstrat^{\ona}$, for the above OOG $\tilde{\ona}$, which is implemented by $\ona$. The proofs of all theorems in this section are to be found in Section \ref{proofsec}\,.

\subsection{The Game $\cod$}
We shall approach the facility location game via the well studied OOG $\gneg{N}{G}$ for some $N\in\mathbb{N}$ and $G\in\rplus$. In this subsection we shall refer to $\gneg{N}{G}$ as $\cod$, which is defined by:
\begin{itemize}
\item $\act{\cod}:=\spx{N}$\,.
\item $\lost{\cod}$ is the set of (differentiable) convex functions $f:[0,1]^{N}\rightarrow\mathbb{R}$ in which $\pder{i}{f}(\bs{x})\in[0,G]$ for all $\bs{x}\in\spx{N}$ and $i\in\na{N}$\,.
\item $\cplx{\cod}(\bs{x}):=0$ for all $\bs{x}\in\spx{N}$\,.
\end{itemize}
We now define the Exponentiated gradient strategy $\bnstrat{\oneg(N,G)}$ for $\cod$. We first define $\lr:=\frac{1}{G}\sqrt{\frac{\log(N)}{\ntr}}$. Given some $\bnac\in\lost{\cod}^{\ntr}$ we define $\nstrat{\oneg(N,G)}{t}({\bnac})$ as follows:

Define $\bwtf{1}:=\bs{1}/N$ and for all $t\in\na{\ntr}$ and $i\in\na{N}$ define:
$$\wtf{t+1}_i:=\frac{\wtf{t}_i\exp(-\lr\pgrd{f_t}{\bwtf{t}}{i})}{\sum_{j\in\na{N}}\wtf{t}_j\exp(-\lr\pgrd{f_t}{\bwtf{t}}{j})}\,.$$
We then define:
$$\nstrat{\oneg(N,G)}{t}({\bnac}):=\dpm{\spx{N}}{\bwtf{t}}\,.$$
 The following theorem is a well known result.
\begin{theorem}\label{t1r2}
The strategy $\bnstrat{\oneg(N,G)}$ has a generalised regret $\genr$ which is bounded as:
$$\genr(\rl,\rk)\leq \rl+G\sqrt{2\ln(N)/\ntr}\,.$$
\end{theorem}

\subsection{The Game $\flo$}

In this section, given some $\fns\in\na{N}$, we consider the OOG $\flo(N, C, D, \fns)$ which is identical to $\fl$ except that:
$$\cplx{\flo(N, C, D, \fns)}(\spf):=\ity\indi{|\spf|\neq\fns}$$ for all $\spf\in\act{\flg}$. In this subsection we will refer to $\flo(N, C, D, \fns)$ as $\flo$. Letting $\tfns:=\fns\lceil\ln(\ntr)/2\rceil$ we will, in this section, create and analyse a transformation $\tro$ from $\flo$ into $\cod(N,(C+D)\tfns)$. We first define two functions. We define a function $\ord:\mathbb{R}^N\times\na{N}\rightarrow\na{N}$ such that, for all $\bdv\in\mathbb{R}^N$, we have:
\begin{itemize}
\item $\{\ord(\bdv,i):i\in\na{N}\}=\na{N}$\,.
\item $\dv{\ord(\bdv,i+1)}\leq\dv{\ord(\bdv,i)}~~~\forall i\in\na{N-1}$\,.
\end{itemize}
and we define the function $\mts{}:\na{N}^{\tfns}\rightarrow\act{\flo}$ such that for all $\biid\in\na{N}^{\tfns}$ we have:
$$\mts{}(\biid):=\{i\in\na{\nsi}~|~\exists j\in\na{\tfns}:\iid{j}=i\}\,.$$
We now define the transformation $\tro$ by:
$$\at{\tro}(\spf)_i:=\indi{i\in\spf}/\fns~~~~~~\forall i\in\na{N},~ \spf\in\act{\flo}:|\spf|=\fns\,.$$
$$\at{\tro}(\spf)_i~\operatorname{is~arbitrary}~~~~~~\forall i\in\na{N},~ \spf\in\act{\flo}:|\spf|\neq\fns\,.$$
$$\pt{\tro}(\bs{w}):=\sum_{\biid\in\na{\nsi}^{\tfns}}\left(\prod_{i\in\na{\tfns}}w_{\iid{i}}\right)\dpm{\act{\flf{\fns}}}{\mts{}(\biid)}\,.$$
$$[\ft{\tro}(\fll{\bocv{}}{\bdv})](\bs{w}):=\tfns\bocv{}\cdot\bs{w}+\dv{\ord(\bdv,N)}+\sum_{i\in\na{N-1}}\left(\dv{\ord(\bdv,i)}-\dv{\ord(\bdv,i+1)}\right)\left(\sum_{j\in\na{i}}w_{\ord(\bdv,j)}\right)^{\tfns}\,.$$
for all $\bs{w}\in\act{\cod(N,(C+D)\tfns)}$ and $(\bocv{},{\bdv})\in[0,C]^N\times[0,D]^N$.
Note that, given $\bs{w}\in\spx{N}$, it is easy to sample from $\pt{\tro}(\bs{w})$: just sample $\tfns$ points $i$ uniformly at random, and with replacement, with probability $w_i$.  The fact that $\tro$ is a true transformation follows from the following two theorems:
\begin{theorem}\label{fnst1}
For all $(\bocv{},{\bdv})\in[0,C]^N\times[0,D]^N$ we have that $\ft{\tro}(\fll{\bocv{}}{\bdv})\in\lost{\co{N,(C+D)\tfns}}$\,.
\end{theorem}
\begin{theorem}\label{fnst2}
For all $(\bocv{},{\bdv})\in[0,C]^N\times[0,D]^N$ and $\bs{w}\in\spx{N}$ we have:
$$[\ft{\tro}(\fll{\bocv{}}{\bdv})](\bs{w})\geq\int_{\act{\flo}}\fll{\bocv{}}{\bdv}~[\pt{\tro}(\bs{w})]\,.$$
\end{theorem}
We have the following theorem:
\begin{theorem}\label{t1r3}
For all $\rl,\rk\in\rplus$ we have $\lm{\tro}(\rl,\rk)\leq\lceil\ln(\ntr)/2\rceil\rl+D\sqrt{1/\ntr}+\rk$\,.
\end{theorem}
We define  $\bstrat^{\onflo(N,C,D,\fns)}:=\stf{\tro}\left(\bnstrat{\oneg(N,(C+D)\tfns)}\right)$. Combining theorems \ref{t1r1}, \ref{t1r2} and \ref{t1r3} gives us:
\begin{theorem}\label{t2r2}
$\bstrat^{\onflo(N,C,D,\fns)}$ has a generalised regret $\gnrs{}{}$ that is bounded by:
$$\gnrs{}{}(\rl,\rk)\leq\lceil\ln(\ntr)/2\rceil\rl+(2\fns\lceil\ln(\ntr)/2\rceil(C+D)+D)\sqrt{\ln(N)/\ntr}+\rk\,.$$
\end{theorem}

\subsection{The Game $\flt$}
In this section, given some $\fns\in\na{N}$, we consider the OOG $\flt(N, C, D, \fns)$ which is identical to $\fl$ except that:
$$\cplx{\flt(N, C, D, \fns)}(\spf):=\ity\indi{|\spf|>\fns}$$ for all $\spf\in\act{\flg}$. In this subsection we will refer to $\flt(N, C, D, \fns)$ as $\flt$.  We will  now analyse a transformation $\trt$ from $\flt$ into $\flo(2N, C, C+D)$ which is defined as follows:
$$\at{\trt}(\spf):=\spf\cup\{N+i:i\leq\fns-|\spf|\}~~~~~~\forall \spf\in\act{\flt}:|\spf|\leq\fns\,.$$
$$\at{\trt}(\spf)~\operatorname{is~arbitrary}~~~~~~\forall \spf\in\act{\flt}:|\spf|>\fns\,.$$
$$\pt{\trt}(\spf'):=\dpm{\act{\flb{\fns}}}{\spf'\cap\na{N}}~~~~~~\forall \spf'\in\act{\flo(2N, C, C+D)}:\spf'\cap\na{N}\neq\emptyset\,.$$
$$\pt{\trt}(\spf'):=\dpm{\act{\flb{\fns}}}{\{1\}}~~~~~~\forall \spf'\in\act{\flo(2N, C, C+D)}:\spf'\cap\na{N}=\emptyset\,.$$
$$\ft{\trt}(\fll{\bocv{}}{\bdv})=\fll{\hat{\bocv{}}}{\hat{\bdv}}~~~~~~\forall\bocv{},\bdv\in\mathbb{R}^N\,.$$
where $\hat{\bocv{}},\hat{\bdv}\in\mathbb{R}^{2N}$ are defined so that:
$$\hat{c}_{i}:=c_i,~ ~~\hat{d}_i:=d_i~~~~~~\forall i\in\na{N}\,.$$
$$\hat{c}_i:=0, ~~~\hat{d}_i:=C+D~~~~~~\forall i\in\na{2N}\setminus\na{N}\,.$$
The following theorem asserts that $\trt$ is a genuine transformation:
\begin{theorem}\label{btt}
$\trt$ is a transformation from $\flt(N, C, D, \fns)$ into $\flo(2N, C, C+D,\fns)$\,.
\end{theorem}
We also have the following theorem:
\begin{theorem}\label{t2r1}
For all $\rl,\rk\in\rplus$ we have $\lm{\trt}(\rl,\rk)\leq\rl+\rk$\,.
\end{theorem}
We define  $\bstrat^{\onflt(N, C, D, \fns)}:=\stf{\trt}\left(\bnstrat{\flo(2N, C, C+D)}\right)$. Combining theorems \ref{t1r1}, \ref{t2r2} and \ref{t2r1} gives us:
\begin{theorem}\label{bflth}
$\bstrat^{\operatorname{\onflt(N, C, D, \fns)}}$ has a generalised regret $\gnrs{}{}$ that is bounded by:
$$\gnrs{}{}(\rl,\rk)\leq\lceil\ln(\ntr)/2\rceil\rl+(2\fns\lceil\ln(\ntr)/2\rceil(2C+D)+(C+D))\sqrt{\ln(2N)/\ntr}+\rk\,.$$
\end{theorem}

\subsection{The Game $\flg$}\label{flgs}
In this subsection we will refer to the OOG $\flg(N,C,D)$ as $\flg$\,.

By considering the game $\flt(N, C, D, N)$ we automatically have that the strategy $\bstrat^{\operatorname{\onflt(N, C, D, N)}}$ gives us a generalised regret $\gnrs{}{}$, for the game $\flg$, that is bounded by:
$$\gnrs{}{}(\rl,\rk)\leq\lceil\ln(\ntr/2)\rceil\rl+(2N\lceil\ln(\ntr)/2\rceil(2C+D)+(C+D))\sqrt{\ln(2N)/\ntr}\,.$$

Utilising the doubling trick of Subsection \ref{GDTss} on the game $\flt(N, C, D, \fns)$  gives us a strategy, $\bstrat^{\onfl(N,C,D)}$, for the game $\flg(N,C,D)$, with generalised regret $\genr$ bounded by
$$\genr(\rl,\rk)\in\mathcal{O}\left(\rl\ln(\ntr) +\rk(\mcos+\mdis)\ln(T)\sqrt{\frac{\ln(\nsi)}{T}} \right)\,.$$

Specifically, we define $a:=\lceil\ln(\ntr/2)\rceil(4C+2D)$, $b:=C+D$ and the OOG $\g$, which appears in Subsection \ref{GDTss}, the same as $\flg$ except that $\cplx{\g}(\spf):=(a\fns+b)/(a+b)$. For all $\trsh\geq 1$ we then define the strategy $\bstrat^{\trsh}$, appearing in Subsection \ref{GDTss}, as equal to $\bstrat^{\operatorname{\onflt(N, C, D, \lceil(\trsh(a+b)-b)/a\rceil)}}$. Combining theorems \ref{dtth} and \ref{bflth} gives us the following theorem:

\begin{theorem}\label{finth}
Define $a:=\lceil\ln(\ntr)/2\rceil(4C+2D)$, $b:=C+D$, and the OOG $\g$ as the same as $\flg$ except that $\cplx{\g}(\spf):=(a\cplx{\flg}(\spf)+b)/(a+b)$ for all $\fns\in\act{\flg(N,C,D)}$. Also define $\bstrat^{\trsh}:=\bstrat^{\operatorname{\onflt(N, C, D, \lceil(\trsh(a+b)-b)/a\rceil)}}$ for all $\trsh\geq 1$. Then the strategy $\bstrat^{\dt}$, as defined in Definition \ref{dtdef}, has a generalised regret, with respect to the OOG $\flg(N,C,D)$, of:
$$\gnrs{\flg}{\bstrat^{\dt}}(\rl,\rk)\in\mathcal{O}\left(\rl\ln(\ntr) +\rk(\mcos+\mdis)\ln(T)\sqrt{\frac{\ln(\nsi)}{T}} \right)\,.$$
\end{theorem}

With Theorem \ref{finth} in hand we let our strategy $\bstrat^{\onfl(N,C,D)}$ be equal to $\bstrat^{\dt}$\,.

\section{Proofs}\label{proofsec}

We now prove the theorems in sections \ref{ecs}, \ref{oogsec} and \ref{devstratsec}, in order.

\subsection{Proof of Theorem \ref{fderth}}

By a simple induction we have, for all $i\in\na{N-1}$, $s_i=\sum_{j\in\na{i}}w_{\ord(j)}$. This immediately gives us $\lambda=f(\bs{w})$. Also, this gives us, via another induction, that, for all $i\in\na{N-1}$:
\begin{align}
\notag s'_i&=\sum_{k= i}^{N-1}\left(\dv{\ord(k)}-\dv{\ord(k+1)}\right)s_{k}^{\tfns-1}\\
\notag&=\sum_{k= i}^{N-1}\left(\dv{\ord(k)}-\dv{\ord(k+1)}\right)\left(\sum_{j\in\na{k}}w_{\ord(j)}\right)^{\tfns-1}
\end{align}
Now, the derivative of $\left(\sum_{j\in\na{k}}w_{\ord(j)}\right)^{\tfns}$ with respect to $w_{\ord(i)}$ is equal to $0$ if $k<i$ and equal to $\tfns\left(\sum_{j\in\na{k}}w_{\ord(j)}\right)^{\tfns-1}$ if $k\geq i$. This means that $\tfns s'_i$ is the derivative of \\$\sum_{k\in\na{N-1}}\left(\dv{\ord(k)}-\dv{\ord(k+1)}\right)\left(\sum_{j\in\na{k}}w_{\ord(j)}\right)^{\tfns}$ with respect to $w_{\ord(i)}$. Since $\tfns\ocv{}{i}$ is the derivative of $\bocv{}\cdot\bs{w}$ with respect to $w_{\ord(i)}$ we then have that $g_{\ord(i)}$ is the derivative of $f(\bs{w})$ with respect to $w_{\ord(i)}$. This completes the proof.

\hfill $\blacksquare$

\subsection{Proof of Theorem \ref{msfsth}}

We utilise the notation defined in Algorithm \ref{algsam}. Given a node $u$, of $\mathcal{B}$ we let $\des{u}$ be the set of leaves of $\mathcal{B}$ which are descendants of $u$.

\begin{lemma}\label{msfsl1}
We have $p'_{v_0}=1$.
\end{lemma}

\begin{proof}
We first prove, via reverse induction on $\delta$ (i.e. from $\delta= H$ to $\delta=0$) then for all nodes $u$ at depth $\delta$ we have $p'_u=\sum_{j\in\des{u}}p_{\tau(j)}$. This is clearly the case when $\delta=H$ because then $u$ is a leaf so $p'_u=p_{\tau(u)}$ and $\des{u}=\{u\}$ so $\sum_{j\in\des{u}}p_j=p_{\tau(u)}$. Suppose the inductive hypothesis holds for $\delta=\delta'$ (for some $\delta'\in\na{H}$). Then, if $u$ is at depth $\delta'-1$, we have that $\lchild{u}$ and $\rchild{u}$ are a depth $\delta'$ so:
\begin{align}
\notag \sum_{j\in\des{u}}p_{\tau(j)}&=\sum_{j\in\des{\lchild{u}\cup\des{\rchild{u}}}}p_{\tau(j)}\\
\notag &=\sum_{j\in\des{\lchild{u}}}p_{\tau(j)}+\sum_{j\in\des{\rchild{u}}}p_{\tau(j)}\\
\notag &=p'_{\lchild{u}}+p'_{\rchild{u}}\\
\notag &=p'_u
 \end{align}
 so the inductive hypothesis holds for $\delta=\delta-1$. This proves the inductive hypothesis and hence that:
\begin{align}
\notag p'_{v_0}&=\sum_{j\in\des{v_0}}p_{\tau(j)}\\
\notag&=\sum_{k\in\na{N'}}p_k\\
\notag&=\sum_{k\in\na{N}}p_k\\
\notag&=1
\end{align}
\end{proof}
 
Now let $l:=\tau^{-1}(i)$ where $i$ is as in the theorem statement. For all $\delta\in\{0\}\cup\na{H}$, let $a_{\delta}$ be the ancestor of $l$ at depth $\delta$. We have the following lemma:

\begin{lemma}\label{msfsl2}
We have:
$$\mathbb{P}(l=v_{H}~|~v_{\delta}=a_{\delta})=\frac{p_i}{p'_{a_{\delta}}}$$
\end{lemma}

\begin{proof}

We prove, via reverse induction on $\delta$ (i.e. from $\delta=H$ to $\delta=0$)
 
When $\delta=H$ we have $a_{\delta}=l$ so:
\begin{align}
\notag\mathbb{P}(l=v_{H}~|~v_{\delta}=a_{\delta})&=\mathcal{P}(l=v_{H}~|~v_{H}=l)\\
\notag&=1\\
\notag&=\frac{p_i}{p_{\tau(l)}}\\
\notag&=\frac{p_i}{p'_{l}}\\
\notag&=\frac{p_i}{p'_{a_{\delta}}}
\end{align}
so the inductive hypothesis holds for $\delta=H$. Now suppose the inductive hypothesis holds for $\delta=\delta'$ (for some $\delta'\in\na{H}$). We now show that it holds for $\delta=\delta'-1$. Firstly, if $a_{\delta'}=\lchild{a_{\delta'-1}}$, we have:
\begin{align}
\notag\mathbb{P}(v_{\delta'}=a_{\delta'}~|~v_{\delta'-1}=a_{\delta'-1})&=\mathbb{P}\left(r_{\delta'-1}\leq \frac{p'_{\lchild{v_{\delta'-1}}}}{p'_{\lchild{v_{\delta'-1}}}+p'_{\rchild{v_{\delta'-1}}}}\right)\\
\notag&=\frac{p'_{\lchild{v_{\delta'-1}}}}{p'_{\lchild{v_{\delta'-1}}}+p'_{\rchild{v_{\delta'-1}}}}\\
\notag&=\frac{p'_{a_{\delta'}}}{p'_{\lchild{v_{\delta'-1}}}+p'_{\rchild{v_{\delta'-1}}}}
\end{align}
and if $a_{\delta'}=\rchild{a_{\delta'-1}}$, we have:
\begin{align}
\notag\mathbb{P}(v_{\delta'}=a_{\delta'}~|~v_{\delta'-1}=a_{\delta'-1})&=\mathbb{P}\left(r_{\delta'-1}> \frac{p'_{\lchild{v_{\delta'-1}}}}{p'_{\lchild{v_{\delta'-1}}}+p'_{\rchild{v_{\delta'-1}}}}\right)\\
\notag&=1-\frac{p'_{\lchild{v_{\delta'-1}}}}{p'_{\lchild{v_{\delta'-1}}}+p'_{\rchild{v_{\delta'-1}}}}\\
\notag&=\frac{p'_{\rchild{v_{\delta'-1}}}}{p'_{\lchild{v_{\delta'-1}}}+p'_{\rchild{v_{\delta'-1}}}}\\
\notag&=\frac{p'_{a_{\delta'}}}{p'_{\lchild{v_{\delta'-1}}}+p'_{\rchild{v_{\delta'-1}}}}
\end{align}
so in either case we have:
\begin{align}
\notag\mathbb{P}(v_{\delta'}=a_{\delta'}~|~v_{\delta'-1}=a_{\delta'-1})&=\frac{p'_{a_{\delta'}}}{p'_{\lchild{v_{\delta'-1}}}+p'_{\rchild{v_{\delta'-1}}}}\\
\notag&=\frac{p'_{a_{\delta'}}}{p'_{v_{\delta'-1}}}\\
\notag&=\frac{p'_{a_{\delta'}}}{p'_{a_{\delta'-1}}}
\end{align}
and hence, by the inductive hypothesis we have:
\begin{align}
\notag&\mathbb{P}(l=v_{H}~|~v_{\delta'-1}=a_{\delta'-1})\\
\notag=&\mathbb{P}(l=v_{H}~|~v_{\delta'}=a_{\delta'}\wedge v_{\delta'-1}=a_{\delta'-1})\mathbb{P}(v_{\delta'}=a_{\delta'}~|~v_{\delta'-1}=a_{\delta'-1})\\
\notag=&\mathbb{P}(l=v_{H}~|~v_{\delta'}=a_{\delta'})\mathbb{P}(v_{\delta'}=a_{\delta'}~|~v_{\delta'-1}=a_{\delta'-1})\\
\notag=&\frac{p_i}{p'_{a_{\delta}}}\mathbb{P}(v_{\delta'}=a_{\delta'}~|~v_{\delta'-1}=a_{\delta'-1})\\
\notag=&\frac{p_i}{p'_{a_{\delta}}}\frac{p'_{a_{\delta'}}}{p'_{a_{\delta'-1}}}\\
\notag=&\frac{p_i}{p'_{a_{\delta'-1}}}
\end{align}
so the inductive hypothesis holds for $\delta=\delta'-1$ and hence holds for all $\delta$.
\end{proof}

 Taking $\delta:=0$ in Lemma \ref{msfsl2}, and noting that the algorithm returns $i$ if and only it $v_{H}=l$ then gives us that the probability of returning $i$ is:
\begin{align}
\notag\mathbb{P}(l=v_{H})&=\mathbb{P}(l=v_{H}~|~v_{0}=a_{0})\\
\notag&=\frac{p_i}{p'_{a_{0}}}\\
\notag&=\frac{p_i}{p'_{v_{0}}}\\
&=p_i \label{indhypeq2}
\end{align}
where Equation \eqref{indhypeq2} comes from Lemma \ref{msfsl1}.

\hfill $\blacksquare$

\subsection{Proof of Theorem \ref{t1r1}}

Given $\bnac\in\lost{\g}^{\ntr}$ we have:
\begin{align}
\notag\avlsn{\stf{\tran}(\bstrat)}{\bnac}&=\frac{1}{\ntri{}}\sum_{t\in\na{\ntri{}}}\ints_{\act{\g}}\nac_t~[[\stf{\tran}(\bstrat)]_t(\bnac)]\\
\notag&=\frac{1}{\ntri{}}\sum_{t\in\na{\ntri{}}}\ints_{\act{\g}}\nac_t\left[\int_{\act{\h}}\pt{\tran}~[\strat{t}(\ft{\tran}(\bnac))]\right]\\
\notag&=\frac{1}{\ntri{}}\sum_{t\in\na{\ntri{}}}\ints_{\act{\g}}\nac_t\int_{\act{\h}}~[\pt{\tran}]~[\strat{t}(\ft{\tran}(\bnac))]\\
\notag&=\frac{1}{\ntri{}}\sum_{t\in\na{\ntri{}}}\ints_{\act{\g}}\int_{\act{\h}}\nac_t~[\pt{\tran}]~[\strat{t}(\ft{\tran}(\bnac))]\\
\notag&=\frac{1}{\ntri{}}\sum_{t\in\na{\ntri{}}}\ints_{\act{\h}}\int_{\act{\g}}\nac_t~[\pt{\tran}]~[\strat{t}(\ft{\tran}(\bnac))]\\
\notag&=\frac{1}{\ntri{}}\sum_{t\in\na{\ntri{}}}\ints_{\act{\h}}\left(\int_{\act{\g}}\nac_t~[\pt{\tran}]\right)~[\strat{t}(\ft{\tran}(\bnac))]\\
\notag&\leq\frac{1}{\ntri{}}\sum_{t\in\na{\ntri{}}}\ints_{\act{\h}}\ft{\tran}(\nac_t)~[\strat{t}(\ft{\tran}(\bnac))]\\
\notag&=\avlsn{\bstrat}{\ft{\tran}(\bnac)}
\end{align}

Now suppose we have $\rl,\rk\in\rplus$. Let:
$$\bar{\bnac}:=\operatorname{argmax}_{\bnac\in\abs{\rl}{\rk}}\{\avlsn{\stf{\tran}(\bstrat)}{\bnac}\}$$
where $\abs{\rl}{\rk}$ is the set of all $\bnac\in\lost{}^{\ntri{}}$ such that there exists $\lac\in\act{}$ with $\cplx{\g}(\lac)\leq\rk$ and $\avlsn{\bcstr{\lac}}{\bnac}\leq\rl$. Note that by the definition of generalised regret, and the above inequaltiy, we have:
$$\gnrs{\g}{\stf{\tran}(\bstrat)}(\rl,\rk)=\avlsn{\stf{\tran}(\bstrat)}{\bar{\bnac}}\leq\avlsn{\bstrat}{\ft{\tran}(\bar{\bnac})}$$
Since $\bar{\bnac}\in\abs{\rl}{\rk}$ choose $\lac\in\act{\g}$ such that with $\cplx{\g}(\lac)\leq\rk$ and $\avlsn{\bcstr{\lac}}{\bnac}\leq\rl$. Since $\cplx{\g}(\lac)\leq\rk$ we have, by definition of a transformation, that $\cplx{\h}(\at{\tran}(\lac))\leq\rk$. By definition of $\lm{\tran}(\rl,\rk)$ and the fact that $\hat{\phi}$ is non-negative and concave in its first argument, we have:
\begin{align}
\notag\avlsn{\bcstr{\at{\tran}(\lac)}}{\ft{\tran}(\bar{\bnac})}&=\frac{1}{\ntr}\sum_{t\in\na{\ntr}}[\ft{\tran}(\bar{\nac}_t)](\at{\tran}(\lac))\\
\notag&\leq\frac{1}{\ntr}\sum_{t\in\na{\ntr}}\lm{\tran}(\nac_t(\lac),\rk)\\
\notag&\leq\frac{1}{\ntr}\sum_{t\in\na{\ntr}}\hphi(\nac_t(\lac),\rk)\\
\notag&\leq\hphi\left(\frac{1}{\ntr}\sum_{t\in\na{\ntr}}\nac_t(\lac),\rk\right)\\
\notag&=\hphi\left(\avlsn{\bcstr{\lac}}{\bnac},\rk\right)\\
\notag&\leq\hphi({\rl},{\rk})
\end{align}
So we have $\at{\tran}(\lac)\in\act{\h}$ such that $\avlsn{\bcstr{\at{\tran}(\lac)}}{\ft{\tran}(\bar{\bnac})}\leq\hphi(\rl,\rk)$ and $\cplx{\h}(\at{\tran}(\lac))\leq\rk$. By definition of generalised regret, we then have:
$$\avlsn{\bstrat}{\ft{\tran}(\bar{\bnac})}\leq\gnrs{\h}{\bstrat}(\hphi(\rl,\rk),\rk)$$
Combining with the above inequality that $\gnrs{\g}{\stf{\tran}(\bstrat)}(\rl,\rk)\leq\avlsn{\bstrat}{\ft{\tran}(\bar{\bnac})}$ gives us the result.

\hfill $\blacksquare$

\subsection{Proof of Theorem \ref{dtth}}

We now analyse the strategy $\bstrat^{\dt}$. First, let $j$ be such that $\ttr{\ntri{}+1}=\nosg 2^j$ and let $\lac$ be an arbitrary element of $\act{}$.  For all $i\in\na{j}\cup\{0\}$ let $s_i$ be the first trial $t$ on which $\ttr{t}=2^i$. We define $s_{j+1}:=\ntr+1$ Note that, for all $i\in\na{j}\cup\{0\}$, we have that:
$$\{t\in\ntri{}~|~\ttr{t}=2^i\}=\{t\in\ntri~|~s_i\leq t<s_{i+1}\}$$
and for all $t'\in\{t\in\ntri{}~|~\ttr{t}=2^i\}$ we have $\tatr{t'}=s_i-1$. 

 We start with the following lemma:

\begin{lemma}\label{endeq1l}
For all $i\leq j$ we have:
$$\sum_{t=s_i}^{s_{i+1}-1}\ints_{\act{}}\nac_{t}~[\strat{t}^{\dt}(\bnac)] \leq \ntr\avlsn{\bstrat^{2^i}}{\shift{\bnac}{s_i-1, s_{i+1}-1}}$$
\end{lemma}

\begin{proof}
\begin{align}
\notag\sum_{t=s_i}^{s_{i+1}-1}\ints_{\act{}}\nac_{t}~[\strat{t}^{\dt}(\bnac)] &=\sum_{t=s_i}^{s_{i+1}-1}\ints_{\act{}}\nac_{t}~[\strat{t+1-\tatr{t}}^{\ttr{t}}(\shift{\bnac}{{\tatr{t},\ntr}})]\\
\notag&=\sum_{t=s_i}^{s_{i+1}-1}\ints_{\act{}}\nac_{t}~[\strat{t-s_i}^{\ttr{t}}(\shift{\bnac}{{{s_i-1,~\ntr}}})]\\
\label{chsheq}&=\sum_{t=s_i}^{s_{i+1}-1}\ints_{\act{}}\nac_{t}~[\strat{t-s_i}^{\ttr{t}}(\shift{\bnac}{{{s_i-1,~s_{i+1}-1}}})]\\
\notag&=\sum_{t=s_i}^{s_{i+1}-1}\ints_{\act{}}\shift{\nac_{t-s_i+1}}{s_i-1,~\ntr}~[\strat{t-s_i}^{\ttr{t}}(\shift{\bnac}{{{s_i-1,~s_{i+1}-1}}})]\\
\notag&=\sum_{t=s_i}^{s_{i+1}-1}\ints_{\act{}}\shift{\nac_{t-s_i+1}}{s_i-1,~s_{i+1}-1}~[\strat{t-s_i}^{\ttr{t}}(\shift{\bnac}{{{s_i-1,~s_{i+1}-1}}})]\\
\notag&\leq\sum_{t=s_i}^{\ntr}\ints_{\act{}}\shift{\nac_{t-s_i+1}}{s_i-1,~s_{i+1}-1}~[\strat{t-s_i}^{\ttr{t}}(\shift{\bnac}{{{s_i-1,~s_{i+1}-1}}})]\\
\notag&\leq\sum_{t=s_i}^{\ntr+s_i-1}\ints_{\act{}}\shift{\nac_{t-s_i+1}}{s_i-1,~s_{i+1}-1}~[\strat{t-s_i}^{\ttr{t}}(\shift{\bnac}{{{s_i-1,~s_{i+1}-1}}})]\\
\notag&=\sum_{t'=1}^{\ntr}\ints_{\act{}}\shift{\nac_{t'}}{s_i-1,~s_{i+1}-1}~[\strat{t'}^{\ttr{t}}(\shift{\bnac}{{{s_i-1,~s_{i+1}-1}}})]\\
\notag&=\sum_{t'=1}^{\ntr}\ints_{\act{}}\shift{\nac_{t'}}{s_i-1,~s_{i+1}-1}~[\strat{t'}^{2^i}(\shift{\bnac}{{{s_i-1,~s_{i+1}-1}}})]\\
\notag&=\ntr\avlsn{\bstrat^{2^i}}{\shift{\bnac}{s_i-1, s_{i+1}-1}}
\end{align}
We Equation \eqref{chsheq} comes from the fact that $\strat{t'}^{\ttr{t}}(\bnac')$ is independent of $\nac'_{t''}$ for all $t''>t'$ (for any $t'\in\ntr$ and $\bnac\in\lost{}^{\ntr}$).
\end{proof}

\begin{lemma}\label{endeq2l}
Given $i$ is such that $i\leq j$ and $\cplx{\g}(\lac)\leq 2^i$, we have:
$$\ntr\avlsn{\bstrat^{2^i}}{\shift{\bnac}{s_i-1, s_{i+1}-1}}\leq\nosg 2^i\ntr+\nosf\sum_{t'=s_i}^{s_{i+1}-1}\nac_{t'}(x)$$
\end{lemma}

\begin{proof}
We have $\cplx{\g^{2^i}}(\lac)=\cplx{\g}^{2^i}(\lac)=0$ so, by the generalised regret of $\bstrat^{2^i}$ we have:
\begin{align}
\notag&\ntr\avlsn{\bstrat^{2^i}}{\shift{\bnac}{s_i-1, s_{i+1}-1}}\\
\notag\leq& \ntr\gnrs{\g^{2^i}}{\bstrat^{2^i}}\left(\avlsn{\bcstr{\lac}}{\shift{\bnac}{s_i-1, s_{i+1}-1}},0\right)\\
\notag\leq& \ntr\nosf\avlsn{\bcstr{\lac}}{\shift{\bnac}{s_i-1, s_{i+1}-1}}+\ntr\nosg 2^i\\
\notag=&\ntr\nosg 2^i+\nosf\sum_{t=1}^T\ints_{\act{}}\shift{\nac_{t}}{s_i-1, s_{i+1}-1}~[\dpm{\act{}}{\lac}]\\
\notag=&\ntr\nosg 2^i+\nosf\sum_{t=1}^T\shift{\nac_{t}}{s_i-1, s_{i+1}-1}(x)\\
\notag=&\ntr\nosg 2^i+\nosf\sum_{t=1}^{(s_{i+1}-1)-(s_i-1)}\shift{\nac_{t}}{s_i-1, s_{i+1}-1}(x)+\nosf\sum_{t=(s_{i+1}-1)-(s_i-1)+1}^{T}\shift{\nac_{t}}{s_i-1, s_{i+1}-1}(x)\\
\notag=&\ntr\nosg 2^i+\nosf\sum_{t=1}^{s_{i+1}-s_i}\shift{\nac_{t}}{s_i-1, s_{i+1}-1}(x)+\nosf\sum_{t=(s_{i+1}-1)-(s_i-1)+1}^{T}\shift{\nac_{t}}{s_i-1, s_{i+1}-1}(x)\\
\notag=&\ntr\nosg 2^i+\nosf\sum_{t=1}^{s_{i+1}-s_i}\shift{\nac_{t}}{s_i-1, s_{i+1}-1}(x)+\nosf\sum_{t=T-(s_{i+1}-1)-(s_i-1)+1}^{T}\zfun(x)\\
\notag=&\ntr\nosg 2^i+\nosf\sum_{t=1}^{s_{i+1}-s_i}\shift{\nac_{t}}{s_i-1, s_{i+1}-1}(x)\\
\notag=&\ntr\nosg 2^i+\nosf\sum_{t=1}^{s_{i+1}-s_i}\nac_{t+s_i-1}(x)\\
\notag=&\ntr\nosg 2^i+\nosf\sum_{t'=s_i}^{s_{i+1}-1}\nac_{t'}(x)
\end{align}

\end{proof}

Combining lemmas \ref{endeq1l} and \ref{endeq2l} gives us the following lemma:
\begin{lemma}\label{leme6com}
We have:
$$\sum_{t=s_i}^{s_{i+1}-1}\ints_{\act{}}\nac_{t}~[\strat{t}^{\dt}(\bnac)] \leq \nosg 2^i\ntr+\nosf\sum_{t=s_i}^{s_{i+1}-1}\nac_{t}(x)$$
for all $i\in\mathbb{N}$ with $2^i\geq\cplx{\g^{\trsh}}(\lac)$ and $i\leq j$. 
\end{lemma}

\begin{proof}
Direct from lemmas \ref{endeq1l} and \ref{endeq2l}
\end{proof}

\begin{lemma}\label{dtlem1}
For any $k\leq j$ with $2^k\geq\cplx{\g^{\trsh}}(\lac)$ we have:
$$\sum_{t=s_k}^{\ntr}\ints_{\act{}}\nac_{t}~[\strat{t}^{\dt}(\bnac)]\leq\nosg(2^{j+1}-2^k)\ntr+\nosf\ntr\avlsn{\bcstr{\lac}}{\bnac}$$
\end{lemma}

\begin{proof}
\begin{align}
\notag\sum_{t=s_k}^{\ntr}\ints_{\act{}}\nac_{t}~[\strat{t}^{\dt}(\bnac)]&=\sum_{i=k}^j\sum_{t=s_i}^{s_{i+1}-1}\ints_{\act{}}\nac_{t}~[\strat{t}^{\dt}(\bnac)]\\
\label{E32eq2}&=\sum_{i=k}^j \left(\nosg2^i\ntr+\nosf\sum_{t=s_i}^{s_{i+1}-1}\nac_{t}(x)\right)\\
\notag&=\nosg(2^{j+1}-2^k)\ntr+\nosf\sum_{i=k}^j \sum_{t=s_i}^{s_{i+1}-1}\nac_{t}(x)\\
\notag&=\nosg(2^{j+1}-2^k)\ntr+\nosf\sum_{t=s_k}^{\ntr}\nac_{t}(x)\\
\notag&\leq\nosg(2^{j+1}-2^k)\ntr+\nosf\sum_{t=1}^{\ntr}\nac_{t}(x)\\
\notag&=\nosg(2^{j+1}-2^k)\ntr+\nosf\ntr\avlsn{\bcstr{\lac}}{\bnac}
\end{align}
where Equation \eqref{E32eq2} comes from Lemma \ref{leme6com}
\end{proof}

\begin{lemma}\label{Eeq33l}
For all $i\leq j$ we have:
$$\ltr{s_{i+1}-1}=\sum_{t=s_i}^{s_{i+1}-1}\ints_{\act{}}\nac_{t}~[\strat{t}^{\dt}(\bnac)]$$
where $\ltr{s_{i+1}-1}$ is as in Definition \ref{dtdef}.
\end{lemma}

\begin{proof}
From Definition \ref{dtdef} we have $\ltr{s_i}=\ints_{\act{}}\nac_{s_i}~[\strat{s_i}^{\dt}(\bnac)]$ and for all $t\in\na{\ntr}$ with $s_i\leq t< s_{i+1}$ we have $\ltr{t+1}:=\ltr{t}+\ints_{\act{}}\nac_{t+1}~[\strat{t+1}^{\dt}(\bnac)]$ so by induction we have the result.
\end{proof}

\begin{lemma}\label{dtlem2}
For any $k\leq j+1$  We have:
$$\sum_{t=1}^{s_{k}-1}\ints_{\act{}}\nac_{t}~[\strat{t}^{\dt}(\bnac)]\leq 2\nosg 2^k\ntr+\sum_{i=1}^{k-1}\maxl{2^i}$$
\end{lemma}

\begin{proof}
Let $i$ be an arbitrary number such that $i<k$. Since, by Definition \ref{dtdef}, we must have that $\ltr{s_{i+1}-2}<2\nosg\ttr{s_{i+1}-2}\ntr=2\nosg2^i\ntr$ we must have:
\begin{align}
\notag\ltr{s_{i+1}-1}&<2\nosg 2^i\ntr+\ints_{\act{}}\nac_{t}~[\strat{s_{i+1}-1}^{\dt}(\bnac)]\\
\notag&=2\nosg2^i\ntr+\ints_{\act{}}\nac_{t}~\strat{s_{i+1}-1-\tatr{t}}^{2^i}(\shift{\bnac}{{s_{i+1}-1-\tatr{t},\ntr}})\\
\notag&\leq 2\nosg 2^i \ntr+ \maxl{2^i}
\end{align}
Substituting into the equality of Lemma \ref{Eeq33l} gives us:
\begin{align}
\notag\sum_{t=s_i}^{s_{i+1}-1}\ints_{\act{}}\nac_{t}~[\strat{t}^{\dt}(\bnac)]<2\nosg 2^i\ntr + \maxl{2^i}
\end{align}
so:
\begin{align}
\notag\sum_{t=1}^{s_{k}-1}\ints_{\act{}}\nac_{t}~[\strat{t}^{\dt}(\bnac)]&=\sum_{i=1}^{k-1}\sum_{t=s_i}^{s_{i+1}-1}\ints_{\act{}}\nac_{t}~[\strat{s_{i+1}-1}^{\dt}(\bnac)]\\
\notag&=\sum_{i=1}^{k-1}\left(2\nosg2^i \ntr+ \maxl{2^i}\right)\\
\notag&\leq 2\nosg 2^k\ntr+\sum_{i=1}^{k-1}\maxl{2^i}
\end{align}
\end{proof}

\begin{lemma}\label{dtlem3}
If $\cplx{\g}(\lac)\leq 2^{j-1}$ We have:
$$\nosg 2^{j-1}\leq \nosf\avlsn{\bcstr{\lac}}{\bnac}$$
\end{lemma}

\begin{proof}
From Definition \ref{dtdef} we have $\ltr{s_j-1}\geq 2\cdot 2^{j-1}\nosg\ntr$ so, by lemmas \ref{leme6com} and \ref{Eeq33l} we have:
\begin{align}
\notag2^j\nosg\ntr&\leq\ltr{s_j-1}\\
\notag&=\sum_{t=s_{j-1}}^{s_{j}-1}\ints_{\act{}}\nac_{t}~[\strat{t}^{\dt}(\bnac)]\\
\notag&\leq\nosg 2^{j-1}\ntr+\nosf\sum_{t=s_{j-1}}^{s_{j}-1}\nac_{t}(\lac)\\
\notag&\leq\nosg 2^{j-1}\ntr+\nosf\sum_{t=1}^{\ntr}\nac_{t}(\lac)\\
\notag&=\nosg 2^{j-1}\ntr+\nosf\sum_{t=1}^T\ints_{\act{}}\nac_t~[\dpm{\act{}}{\lac}]\\
\notag&=\nosg 2^{j-1}\ntr+\nosf\ntr\avlsn{\bcstr{\lac}}{\bnac}
\end{align}
so $\nosg 2^{j-1}=(2-1)\nosg 2^{j-1}=\nosg 2^j-\nosg 2^{j-1}\leq\nosf\avlsn{\bcstr{\lac}}{\bnac}$
\end{proof}

\begin{lemma}\label{dtlem4}
If $\cplx{\g}(\lac)\leq 2^{j-1}$ then:
$$\avlsn{\bstrat^{\dt}}{\bnac}\leq 5\nosf\avlsn{\bcstr{\lac}}{\bnac}+2\nosg\cplx{\g}(\lac)+\frac{1}{\ntr}\sum_{i=1}^{k-1}\maxl{2^i}$$
where $k:=\min\{k\leq j~|~\cplx{\g}(\lac)\leq 2^{k}\}$
\end{lemma}

\begin{proof}
 Combining lemmas \ref{dtlem1}, \ref{dtlem2} and \ref{dtlem3} gives us:
\begin{align}
\notag\avlsn{\bstrat^{\dt}}{\bnac}&=\frac{1}{\ntr}\sum_{t=1}^{\ntr}\ints_{\act{}}\nac_{t}~[\strat{t}^{\dt}(\bnac)]\\
\notag&=\frac{1}{\ntr}\sum_{t=1}^{s_{k}-1}\ints_{\act{}}\nac_{t}~[\strat{t}^{\dt}(\bnac)]+\frac{1}{\ntr}\sum_{s_k}^{\ntr}\ints_{\act{}}\nac_{t}~[\strat{t}^{\dt}(\bnac)]\\
\notag&\leq 2\nosg 2^k+\frac{1}{\ntr}\sum_{i=1}^{k-1}\maxl{2^i}+\nosg((2^{j+1}-2^k)+\nosf\avlsn{\bcstr{\lac}}{\bnac})\\
\notag&=\nosf\avlsn{\bcstr{\lac}}{\bnac}+\nosg2^k+\nosg2^{j+1}+\frac{1}{\ntr}\sum_{i=1}^{k-1}\maxl{2^i}\\
\notag&\leq\nosf\avlsn{\bcstr{\lac}}{\bnac}+\nosg2^k+4\nosf\avlsn{\bcstr{\lac}}{\bnac}+\frac{1}{\ntr}\sum_{i=1}^{k-1}\maxl{2^i}
\end{align}
which, since $2^{k-1}<\cplx{\g}(\lac)$, is bounded above by:
\begin{equation}
\notag 5\nosf\avlsn{\bcstr{\lac}}{\bnac}+2\nosg\cplx{\g}(\lac)+\frac{1}{\ntr}\sum_{i=1}^{k-1}\maxl{2^i}
\end{equation}
\end{proof}

\begin{lemma}\label{dtlem5}
If $\cplx{\g}(\lac)>2^{j-1}$ then:
$$\avlsn{\bstrat^{\dt}}{\bnac}<8\nosg\cplx{\g}(\lac)+\frac{1}{\ntr}\sum_{i=1}^{j}\maxl{2^i}$$
\end{lemma}

\begin{proof}
By Lemma \ref{dtlem2} we have:
\begin{align}
\notag\avlsn{\bstrat^{\dt}}{\bnac}&=\frac{1}{\ntr}\sum_{t=1}^{\ntr}\ints_{\act{}}\nac_{t}~[\strat{t}^{\dt}(\bnac)]\\
\notag&=\frac{1}{\ntr}\sum_{t=1}^{s_{j+1}-1}\ints_{\act{}}\nac_{t}~[\strat{t}^{\dt}(\bnac)]\\
\notag&\leq 2\nosg 2^{j+1}+\frac{1}{\ntr}\sum_{i=1}^{j}\maxl{2^i}\\
\notag&=8\nosg 2^{j-1}+\frac{1}{\ntr}\sum_{i=1}^{j}\maxl{2^i}\\
\notag&<8\nosg\cplx{\g}(\lac)+\frac{1}{\ntr}\sum_{i=1}^{j}\maxl{2^i}
\end{align}
\end{proof}

\begin{lemma}
We have:
$$\avlsn{\bstrat^{\dt}}{\bnac}\leq 5\nosf\avlsn{\bcstr{\lac}}{\bnac}+8\nosg\cplx{\g}(\lac)+\frac{1}{\ntr}\sum_{i=1}^{\lceil\log_2(\cplx{\g}(\lac))\rceil}\maxl{2^i}$$
\end{lemma}

\begin{proof}
Direct from lemmas \ref{dtlem4} and \ref{dtlem5}
\end{proof}

So we have shown that for any $\lac\in\act{}$ and $\bnac\in\lost{}^{\ntr}$ we have: $$\avlsn{\bstrat^{\dt}}{\bnac}\leq 5\nosf\avlsn{\bcstr{\lac}}{\bnac}+8\nosg\cplx{\g}(\lac)+\frac{1}{\ntr}\sum_{i=1}^{\lceil\log_2(\cplx{\g}(\lac))\rceil}\maxl{2^i}$$ Theorem \ref{dtth} follows.

\hfill $\blacksquare$

\subsection{Proof of Theorem \ref{t1r2}}

Given vectors $\bs{a},\bs{b}\in\spx{N}$ we let $\rele{\bs{a},\bs{b}}$ be the relative entropy between $\bs{a}$ and $\bs{b}$. That is:
$$\rele{\bs{a}}{\bs{b}}=\sum_{i\in\na{N}}a_i\ln\left(\frac{a_i}{b_i}\right)$$
It is a standard result that $\rele{\bs{a}}{\bs{b}}\geq 0$.

Suppose now that we have some $\bnac\in\lost{\cod}^{\ntr}$. Let $\bwst$ be an arbitrary vector in $\spx{N}$ and for all $t\in\ntr$ let $\bwtf{t}$ be as defined in the algorithm (with respect to $\bnac$) .

Since $f_t$ is convex we have, by definition of a convex function:
$$f_t(\bwtf{t})-f_t(\bwst) \leq (\grd{f_t}{\bwtf{t}})\cdot(\bwtf{t}-\bwst)$$
so, by letting $\bgv{t}=\grd{f_t}{\bwtf{t}}$ we have:
\begin{equation}\label{t1r2e1}
\sum_{t\in\ntr}(f_t(\bwtf{t})-f_t(\bwst))\leq\sum_{t\in\ntr}(\bgv{t}\cdot\bwtf{t}-\bgv{t}\cdot\bwst)
\end{equation}
Let: 
$$\nrmc{t}:=\sum_{i\in\na{N}}\wtf{t}_{i}\exp(-\lr\gv{t}{i})$$
Since, for all $t\in\na{\ntr-1}$, we have $\wtf{t+1}_{i}=\wtf{t}_{i}\exp(-\lr\gv{t}{i})/\nrmc{t}$ we obtain:
\begin{align}
\notag\rele{\bwst}{\bwtf{t}}-\rele{\bwst}{\bwtf{t+1}}&=\sum_{i\in\na{N}}\wst{i}\ln\left(\frac{\wtf{t+1}_i}{\wtf{t}_i}\right)\\
\notag&=\sum_{i\in\na{N}}\wst{i}\ln\left(\frac{\exp(-\lr\gv{t}{i})}{\nrmc{t}}\right)\\
\notag&=-\lr\sum_{i\in\na{N}}\wst{i}\gv{t}{i}-\sum_{i\in\na{N}}\wst{i}\ln(\nrmc{t})\\
\notag&=-\lr\bwst\cdot\bgv{t}-\ln(\nrmc{t})\\
\notag&=-\lr\bwst\cdot\bgv{t}-\ln\left(\sum_{i\in\na{N}}\wtf{t}_{i}\exp(-\lr\gv{t}{i})\right)\\
\label{t1r2l1}&\geq -\lr\bwst\cdot\bgv{t}-\ln\left(\sum_{i\in\na{N}}\wtf{t}_{i}\left(1-\lr\gv{t}{i}+\frac{1}{2}\lr^2\left(\gv{t}{i}\right)^2\right)\right)\\
\notag&=-\lr\bwst\cdot\bgv{t}-\ln\left(1-\lr\bwtf{t}\cdot\bgv{t}+\frac{1}{2}\lr^2\sum_{i\in\na{N}}\wtf{t}_i(\gv{t}{i})^2\right)\\
\label{t1r2l2}&=-\lr(\bwst\cdot\bgv{t}-\bwtf{t}\cdot\bgv{t})-\frac{1}{2}\lr^2\sum_{i\in\na{N}}\wtf{t}_i(\gv{t}{i})^2\\
\notag&\geq-\lr(\bwst\cdot\bgv{t}-\bwtf{t}\cdot\bgv{t})-\frac{1}{2}\lr^2\sum_{i\in\na{N}}\wtf{t}_iG^2\\
\notag&=-\lr(\bwst\cdot\bgv{t}-\bwtf{t}\cdot\bgv{t})-\frac{1}{2}\lr^2G^2
\end{align}
where equations \eqref{t1r2l1} and \eqref{t1r2l2} come from the inequalities $\exp(-x)\leq 1-x+x^2/2$ (for $x\geq 0$) and $\ln(1+x)\leq x$ respectively.

So we have:
$$\rele{\bwst}{\bwtf{t}}-\rele{\bwst}{\bwtf{t+1}}\geq -\lr(\bwst\cdot\bgv{t}-\bwtf{t}\cdot\bgv{t})-\frac{1}{2}\lr^2G^2$$
Taking a telescoping sum (over $t\in\na{\ntr}$) gives us:
$$\rele{\bwst}{\bwtf{1}}-\rele{\bwst}{\bwtf{T+1}}\geq\sum_{t\in\na{\ntr}}-\lr(\bwst\cdot\bgv{t}-\bwtf{t}\cdot\bgv{t})-\frac{1}{2}\lr^2G^2\ntr$$
so, since relative entropies are positive, we obtain:
$$\rele{\bwst}{\bwtf{1}}\geq\sum_{t\in\na{\ntr}}-\lr(\bwst\cdot\bgv{t}-\bwtf{t}\cdot\bgv{t})-\frac{1}{2}\lr^2G^2\ntr$$
which, upon rearranging and substituting into Equation \eqref{t1r2e1} gives us:
\begin{align}
\notag\sum_{t\in\ntr}(f_t(\bwtf{t})-f_t(\bwst))&\leq\sum_{t\in\ntr}(\bgv{t}\cdot\bwtf{t}-\bgv{t}\cdot\bwst)\\
\notag&\leq\frac{1}{\lr}\rele{\bwst}{\bwtf{1}}+\frac{1}{2}\lr G^2T\\
\notag&=\frac{1}{\lr}\sum_{i\in\na{N}}\wst{i}\ln\left(N\wst{i}\right)+\frac{1}{2}\lr G^2T\\
\notag&\leq\frac{1}{\lr}\sum_{i\in\na{N}}\wst{i}\ln\left(N\right)+\frac{1}{2}\lr G^2T\\
\notag&\leq\frac{1}{\lr}\ln\left(N\right)+\frac{1}{2}\lr G^2T\\
\notag&=G\sqrt{2T\ln(N)}
\end{align}

This implies that:
\begin{align}
\notag\avlsn{\bnstrat{\oneg(N,G)}}{\bnac}&=\frac{1}{T}\sum_{t\in\na{\ntr}}\int_{\spx{N}}f_t~[\nstrat{EG}{t}]\\
\notag&=\frac{1}{T}\sum_{t\in\na{\ntr}}\int_{\spx{N}}f_t~[\dpm{\spx{N}}{\bwtf{t}}]\\
\notag&=\frac{1}{T}\sum_{t\in\na{\ntr}}f_t(\bwtf{t})\\
\notag&=\frac{1}{T}\left(\sum_{t\in\ntr}(f_t(\bwtf{t})-f_t(\bwst))\right)+\frac{1}{T}\sum_{t\in\ntr}f_t(\bwst)\\
\notag&\leq G\sqrt{\frac{2\ln(N)}{T}}+\frac{1}{T}\sum_{t\in\ntr}f_t(\bwst)\\
\notag&=G\sqrt{\frac{2\ln(N)}{T}}+\frac{1}{T}\sum_{t\in\ntr}\int_{\spx{N}}f_t~[\dpm{\spx{N}}{\bwst}]\\
\notag&=G\sqrt{\frac{2\ln(N)}{T}}+\frac{1}{T}\sum_{t\in\ntr}\int_{\spx{N}}f_t~[\cstr{\bwst}{t}]\\
\label{t1r2l3}&\leq G\sqrt{\frac{2\ln(N)}{T}}+\avlsn{\bcstr{\bwst}}{\bnac}
\end{align}

We are now ready to bound the generalised regret. Suppose we have $\rl,\rk\in\rplus$ and assume $\bnac\in\lost{\oneg(N,G)}$ and $\bs{u}\in\act{\oneg(N,G)}$ are such that $\cplx{\oneg(N,G)}(\bs{u})\leq\rk$ and $\avlsn{\bcstr{\bs{u}}}{\bnac}\leq\rl$. By Equation \eqref{t1r2l3} we have:
$$\avlsn{\bnstrat{\oneg(N,G)}}{\bnac}\leq \avlsn{\bcstr{\bwst}}{\bnac}+ G\sqrt{\frac{2\ln(N)}{T}}\leq \rl + G\sqrt{\frac{2\ln(N)}{T}}$$
Maximising across all $\bs{u}$ gives us the result.

\hfill $\blacksquare$

\subsection{Proof of Theorem \ref{fnst1}}

Suppose we have $\bocv{}\in[0,C]$ and $\bdv\in[0,D]$. We first show that $\ft{\tro}(\fll{\bocv{}}{\bdv})$ is convex. This is true since, first, $\sum_{j\in\na{i}}w_{\ord(\bdv,j)}$ is linear (and hence convex) and the function $g:\rplus\rightarrow\rplus$ with $g(x):=x^{\tfns}$ is convex, and hence $\left(\sum_{j\in\na{i}}w_{\ord(\bdv,j)}\right)^{\tfns}$ is a convex function of a convex function and hence  convex. Hence $\sum_{i\in\na{N-1}}\left(\dv{\ord(\bdv,i)}-\dv{\ord(\bdv,i+1)}\right)\left(\sum_{j\in\na{i}}w_{\ord(\bdv,j)}\right)^{\tfns}$ is a positive sum of convex functions and hence convex. Since $\tfns\bocv{}\cdot\bs{w}+\dv{\ord(\bdv,N)}$ is linear and hence convex, we then have that $\ft{\tro}(\fll{\bocv{}}{\bdv})$ is a positive sum of two convex functions and is hence convex.

We now show that $\inorm{\grd{\ft{\tro}(\fll{\bocv{}}{\bdv})}{\bs{w}}}\leq \tfns(C+D)$ for all $\bs{w}\in\spx{N}$ which will complete the proof. We have:
\begin{align}
\notag\pgrd{[\ft{\tro}(\fll{\bocv{}}{\bdv})]}{\bs{w}}{\ord(\bdv,i)}&=\tfns\ocv{}{\ord(\bdv,i)}+\tfns\sum_{j=i}^{N-1}\left(\dv{\ord(\bdv,i)}-\dv{\ord(\bdv,i+1)}\right)\left(\sum_{j\in\na{i}}w_{\ord(\bdv,j)}\right)^{\tfns-1}\\
\notag&\leq\tfns\ocv{}{\ord(\bdv,i)}+\tfns\sum_{j=i}^{N-1}\left(\dv{\ord(\bdv,i)}-\dv{\ord(\bdv,i+1)}\right)\left(\sum_{j\in\na{N}}w_{\ord(\bdv,j)}\right)^{\tfns-1}\\
\notag&\leq\tfns\ocv{}{\ord(\bdv,i)}+\tfns\sum_{j=i}^{N-1}\left(\dv{\ord(\bdv,i)}-\dv{\ord(\bdv,i+1)}\right)1^{\tfns-1}\\
\notag&=\tfns\ocv{}{\ord(\bdv,i)}+\tfns\sum_{j=i}^{N-1}\left(\dv{\ord(\bdv,i)}-\dv{\ord(\bdv,i+1)}\right)\\
\notag&\leq\tfns\ocv{}{\ord(\bdv,i)}+\tfns\dv{\ord(\bdv,i)}\\
\notag&\leq\tfns C+\tfns D
\end{align}

\hfill $\blacksquare$

\subsection{Proof of Theorem \ref{fnst2}}

Suppose we have some $\bocv{}\in[0,C]$, $\bdv\in[0,D]$ and $\bs{w}\in\spx{N}$. For $\spf\in\pows{\na{N}}$ we define $g(\spf):=\sum_{i\in\spf}\ocv{}{i}$ and for all $i\in\na{N}$ we define $h_i(\spf):=\indi{\spf\subseteq\{\ord(\bdv,j):j\in \na{i}\}}$ We start with the following lemma:

\begin{lemma}\label{pal1}
We have:
$$\int_{\act{\onflo}}g~[\pt{\tro}(\bs{w})]\leq\tfns\bs{w}\cdot\boc$$
\end{lemma}

\begin{proof}
We have:
\begin{align}
\notag\int_{\act{\onflo}}g~[\pt{\tro}(\bs{w})]&=\int_{\act{\onflo}}g~d\left[\sum_{\biid\in\na{\nsi}^{\tfns}}\dpm{\act{\flf{\fns}}}{\mts{}(\biid)}\prod_{i\in\na{\tfns}}w_{\iid{i}}\right]\\
\notag&=\sum_{\biid\in\na{\nsi}^{\tfns}}\left(\prod_{i\in\na{\tfns}}w_{\iid{i}}\right)\int_{\act{\flf{\fns}}}g~d\left[\dpm{\act{\flf{\fns}}}{\mts{}(\biid)}\right]\\
\notag&=\sum_{\biid\in\na{\nsi}^{\tfns}}\left(\prod_{i\in\na{\tfns}}w_{\iid{i}}\right)g(\mts{}(\biid))\\
\notag&\leq\sum_{\biid\in\na{\nsi}^{\tfns}}\left(\prod_{i\in\na{\tfns}}w_{\iid{i}}\right)\sum_{j\in\na{\tfns}}\oc{\iid{j}}\\
\notag&=\sum_{j\in\na{\tfns}}\sum_{\biid\in\na{\nsi}^{\tfns}}\left(\prod_{i\in\na{\tfns}}w_{\iid{i}}\right)\oc{\iid{j}}\\
\end{align}
We now analyse the term $\sum_{\biid\in\na{\nsi}^{\tfns}}\left(\prod_{i\in\na{\tfns}}w_{\iid{i}}\right)\oc{\iid{j}}$ for all $j\in\na{N}$. Without loss of generality let $j=N$. Then we have:
\begin{align}
\notag\sum_{\biid\in\na{\nsi}^{\tfns}}\left(\prod_{i\in\na{\tfns}}w_{\iid{i}}\right)\oc{\iid{j}}&=\sum_{\biid\in\na{\nsi}^{\tfns}}\left(\prod_{i\in\na{\tfns}}w_{\iid{i}}\right)\oc{\iid{N}}\\
\notag&=\sum_{\iid{N}\in\na{N}}\sum_{\biid\in\na{N}^{\tfns-1}}\left(\prod_{i\in\na{\tfns}}w_{\iid{i}}\right)\oc{\iid{N}}\\
\notag&=\sum_{\iid{N}\in\na{N}}w_{\iid{N}}\oc{\iid{N}}\sum_{\biid\in\na{N}^{\tfns-1}}\left(\prod_{i\in\na{\tfns-1}}w_{\iid{i}}\right)\\
\notag&=\sum_{\iid{N}\in\na{N}}w_{\iid{N}}\oc{\iid{N}}\left(\prod_{i\in\na{\tfns-1}}\sum_{\iid{i}\in\na{N}}w_{\iid{i}}\right)\\
\notag&=\sum_{\iid{N}\in\na{N}}w_{\iid{N}}\oc{\iid{N}}\left(\prod_{i\in\na{\tfns-1}}1\right)\\
\notag&=\sum_{\iid{N}\in\na{N}}w_{\iid{N}}\oc{\iid{N}}\\
\notag&=\bs{w}\cdot\boc
\end{align}
Substituting into the above gives us the result.
\end{proof}

\begin{lemma}\label{pal2}
For all $i\in\na{N}$ we have:
$$\int_{\act{\onflo}}h_i~[\pt{\tro}(\bs{w})]=\left(\sum_{j\in\na{i}}w_{\ord(\bdv,j)}\right)^{\tfns}$$
\end{lemma}

\begin{proof}
Letting $V:=\{\ord(\bdv,j):j\in \na{i}\}$ we have:
\begin{align}
\notag\int_{\act{\onflo}}h_i~[\pt{\tro}(\bs{w})]&=\int_{\act{\onflo}}h_i~d\left[\sum_{\biid\in\na{\nsi}^{\tfns}}\dpm{\act{\flf{\fns}}}{\mts{}(\biid)}\prod_{i\in\na{\tfns}}w_{\iid{i}}\right]\\
\notag&=\sum_{\biid\in\na{\nsi}^{\tfns}}\left(\prod_{i\in\na{\tfns}}w_{\iid{i}}\right)\int_{\act{\onflo}}h_i~d\left[\dpm{\act{\flf{\fns}}}{\mts{}(\biid)}\right]\\
\notag&=\sum_{\biid\in\na{\nsi}^{\tfns}}\left(\prod_{i\in\na{\tfns}}w_{\iid{i}}\right)h_i(\mts{}(\biid))\\
\notag&=\sum_{\biid\in\na{\nsi}^{\tfns}}\left(\prod_{i\in\na{\tfns}}w_{\iid{i}}\right)\indi{k\in V~~\forall k\in\mts{}(\biid)}\\
\notag&=\sum_{\biid\in\na{\nsi}^{\tfns}}\left(\prod_{i\in\na{\tfns}}w_{\iid{i}}\right)\indi{\iid{j}\in V~~\forall j\in\na{\tfns}}\\
\notag&=\sum_{\biid\in\na{\nsi}^{\tfns}}\left(\prod_{i\in\na{\tfns}}w_{\iid{i}}\right)\indi{\biid\in V^{\tfns}}\\
\notag&=\sum_{\biid\in V^{\tfns}}\left(\prod_{i\in\na{\tfns}}w_{\iid{i}}\right)\\
\notag&=\prod_{i\in\na{\tfns}}\sum_{\iid{i}\in V}w_{\iid{i}}\\
\notag&=\prod_{i\in\na{\tfns}}\sum_{\iid{}\in V}w_{\iid{}}\\
\notag&=\left(\sum_{\iid{}\in V}w_{\iid{}}\right)^\tfns\\
\notag&=\left(\sum_{j\in\na{i}}w_{\ord(\bdv,j)}\right)^{\tfns}
\end{align}
\end{proof}

\begin{lemma}\label{pal3}
For all $\spf\in\act{\flo}$ we have:
$$\fll{\boc}{\bdv}(\spf)=g(\spf)+\dv{\ord(\bdv,N)}+\sum_{i\in\na{N-1}}\left(\dv{\ord(\bdv,i)}-\dv{\ord(\bdv,i+1)}\right)h_i(\spf)$$
\end{lemma}

\begin{proof}
Let $j$ be such that $\ord(\bdv,j)=\operatorname{argmin}_{i\in\spf}\dv{i}$. For all $i<j$ we have that $\ord(\bdv,j)\notin\{\ord(\bdv,k):k\in\na{i}\}$ so since $\ord(\bdv,j)\in\spf$ we have $\spf\not\subseteq\{\ord(\bdv,k):k\in\na{i}\}$ and hence  $h_i(\spf)=0$. On the other hand, for all $k$ such that $\ord(\bdv,k)\in\spf$, we have, by definition of $j$, that $\dv{\ord(\bdv,k)}\geq\dv{\ord(\bdv,j)}$ so, by definition of $\ord$, we have $\ord(\bdv,k)\leq\ord(\bdv,j)$ and so for all $i\geq j$ we have $\ord(\bdv,k)\in\{\ord(\bdv,k'):k'\in\na{i}\}$. Hence, for all $i\geq j$ we have $\spf\subseteq\{\ord(\bdv,k):k\in\na{i}\}$ which implies $h_i(\spf)=1$. Putting together gives us $h_i(\spf)=\indi{i\geq j}$. This implies:
\begin{align}
\notag&\dv{\ord(\bdv,N)}+\sum_{i\in\na{N-1}}\left(\dv{\ord(\bdv,i)}-\dv{\ord(\bdv,i+1)}\right)h_i(\spf)\\
\notag=~&\dv{\ord(\bdv,N)}+\sum_{i\in\na{N-1}}\left(\dv{\ord(\bdv,i)}-\dv{\ord(\bdv,i+1)}\right)\indi{i\geq j}\\
\notag=~&\dv{\ord(\bdv,N)}+\sum_{i=j}^{N-1}\left(\dv{\ord(\bdv,i)}-\dv{\ord(\bdv,i+1)}\right)\\
\notag=~&\dv{\ord(\bdv,N)}+\left(\dv{\ord(\bdv,j)}-\dv{\ord(\bdv,N)}\right)\\
\notag=~&\dv{\ord(\bdv,j)}\\
\notag=~&\min_{i\in\spf}\dv{i}
\end{align}
By definition of $g(\spf)$ we then obtain the result.
\end{proof}

We are now ready to prove the theorem. By Lemma \ref{pal3} we have that:
\begin{align}
\notag&\int_{\act{\onflo}}\fll{\boc}{\bdv}~[\pt{\tro}(\bs{w})]\\
\notag=&\int_{\act{\onflo}}g~[\pt{\tro}(\bs{w})]+\dv{\ord(\bdv,N)}\int_{\act{\onflo}}1~[\pt{\tro}(\bs{w})]+\sum_{i\in\na{N-1}}\left(\dv{\ord(\bdv,i)}-\dv{\ord(\bdv,i+1)}\right)\int_{\act{\onflo}}h_i~[\pt{\tro}(\bs{w})]
\end{align}
Substituting in lemmas \ref{pal1} and \ref{pal2} gives us the result.

\hfill $\blacksquare$

\subsection{Proof of Theorem \ref{t1r3}}

Recall that, by definition of $\lm{\tro}$, we have:

$$\lm{\tro}(\rl,\rk):=\max\{[\ft{\tro}(f)](\at{\tro}(x))~|~(x,f)\in\act{\flo}\times\lost{\flo},~f(x)\leq\rl,~\cplx{\flo}(x)\leq\rk\}$$

Suppose we have $\rl,\rk\in\rplus$. When $\rk\geq\ity$ we trivially have that $[\ft{\tro}(f)](\at{\tro}(\spf))\leq\ity$ for all $\spf\in\act{\flo}$ and $f\in\lost{\flo}$ so: $$\lm{\tro}(\rl,\rk)\leq\ity\leq\rk\leq\lceil\ln(\ntr)/2\rceil\rl+D\sqrt{1/\ntr}+\rk$$

Now let's consider the case that $\rk<\ity$. Suppose we have some $(X,f)\in\act{\flo}\times\lost{\flo}$ with $f(x)\leq\rl$, and $\cplx{\flo}(x)\leq\rk$. Let $\bocv{}\in[0,C]$ and $\bdv\in[0,D]$ be such that $\fll{\bocv{}}{\bccm{}}=f$ and let $\bs{w}=\at{\tro}(\spf)$. Since $\cplx{\fns}(\spf)< \ity$ we have $|\spf|=\fns$ and hence also $w_i:=\indi{i\in\spf}/\fns$. Let $k$ be such that $\ord(\bdv,k)=\operatorname{argmin}_{i\in\spf}\dv{i}$. For all $i\in[N-1]$ we have:
$$\left(\sum_{j\in\na{i}}w_{\ord(\bdv,j)}\right)^{\tfns}\leq\left(\sum_{j\in\na{M}}w_{\ord(\bdv,j)}\right)^{\tfns}=1^{\tfns}=1$$
and for $i\in\na{k-1}$ we have:
\begin{align}
\notag\left(\sum_{j\in\na{i}}w_{\ord(\bdv,j)}\right)^{\tfns}&\leq\left(\sum_{j\in\na{N}\setminus\{k\}}w_{\ord(\bdv,j)}\right)^{\tfns}\\
\notag&=\left(\left(\sum_{j\in\na{N}}w_{\ord(\bdv,j)}\right)-w_{\ord(\bdv,k)}\right)^{\tfns}\\
\notag&=\left(1-1/\fns\right)^{\tfns}\\
\notag&\leq\exp(-1/\fns)^{\tfns}\\
\notag&=\exp(-1/\fns)^{\fns\lceil\ln(\ntr)/2\rceil}\\
\notag&=\exp(-\lceil\ln(\ntr)/2\rceil)\\
\notag&\leq\exp(-\ln(\ntr)/2)\\
\notag&= \sqrt{1/\ntr}
\end{align}
Substituting both these inequalities into the definition of $[\ft{\tro}(\fll{\bocv{}}{\bdv})](\bs{w})$ gives us
\begin{align}
\notag&[\ft{\tro}(\fll{\bocv{}}{\bdv})](\bs{w})\\
\notag=&\tfns\bocv{}\cdot\bs{w}+\dv{\ord(\bdv,N)}+\sum_{i\in\na{N-1}}\left(\dv{\ord(\bdv,i)}-\dv{\ord(\bdv,i+1)}\right)\left(\sum_{j\in\na{i}}w_{\ord(\bdv,j)}\right)^{\tfns}\\
\notag\leq&\tfns\bocv{}\cdot\bs{w}+\dv{\ord(\bdv,N)}+\sum_{i=1}^{k-1}\left(\dv{\ord(\bdv,i)}-\dv{\ord(\bdv,i+1)}\right)\sqrt{1/\ntr}+\sum_{i=k}^{N-1}\left(\dv{\ord(\bdv,i)}-\dv{\ord(\bdv,i+1)}\right)\\
\notag&=\tfns\bocv{}\cdot\bs{w}+\sum_{i=1}^{k-1}\left(\dv{\ord(\bdv,i)}-\dv{\ord(\bdv,i+1)}\right)\sqrt{1/\ntr}+\dv{\ord(\bdv,k)}\\
\notag&=\tfns\bocv{}\cdot\bs{w}+\sqrt{1/\ntr}\sum_{i=1}^{k-1}\left(\dv{\ord(\bdv,i)}-\dv{\ord(\bdv,i+1)}\right)+\dv{\ord(\bdv,k)}\\
\notag&=\tfns\bocv{}\cdot\bs{w}+\sqrt{1/\ntr}\left(\dv{\ord(\bdv,1)}-\dv{\ord(\bdv,k)}\right)+\dv{\ord(\bdv,k)}\\
\notag&\leq\tfns\bocv{}\cdot\bs{w}+\dv{\ord(\bdv,1)}\sqrt{1/\ntr}+\dv{\ord(\bdv,k)}\\
\notag&\leq\tfns\bocv{}\cdot\bs{w}+D\sqrt{1/\ntr}+\dv{\ord(\bdv,k)}\\
\notag&\leq\tfns\bocv{}\cdot\bs{w}+D\sqrt{1/\ntr}+\min_{i\in\spf}\dv{i}\\
\notag&=\tfns\sum_{i\in\spf}\oc{i}/\fns+D\sqrt{1/\ntr}+\min_{i\in\spf}\dv{i}\\
\notag&\leq \fns\lceil\ln(\ntr)/2\rceil\sum_{i\in\spf}\oc{i}/\fns+D\sqrt{1/\ntr}+\min_{i\in\spf}\dv{i}\\
\notag&\leq \lceil\ln(\ntr)/2\rceil\sum_{i\in\spf}\oc{i}+D\sqrt{1/\ntr}+\min_{i\in\spf}\dv{i}\\
\notag&\leq \lceil\ln(\ntr)/2\rceil\left(\min_{i\in\spf}\dv{i}+\sum_{i\in\spf}\oc{i}\right) + D\sqrt{1/\ntr}\\
\notag&=\lceil\ln(\ntr)/2\rceil\fll{\bocv{}}{\bccm{}}(\spf)+D\sqrt{1/\ntr}
\end{align}

So: 
\begin{align}
\notag[\ft{\tro}(f)](\at{\tro}(\spf))&\leq\lceil\ln(\ntr)/2\rceil\fll{\bocv{}}{\bccm{}}(\spf)+D\sqrt{1/\ntr}+\rk\\
\notag&\leq\lceil\ln(\ntr)/2\rceil \rl+D\sqrt{1/\ntr}+\rk
\end{align}
Since this holds for any $(X,f)\in\spf\in\act{\flo}\times\lost{\flo}$ with $f(x)\leq\rl$ and $\cplx{\flo}(x)\leq\rk$, we have:
$$\lm{\tro}(\rl,\rk)\leq\lceil\ln(\ntr)/2\rceil\rl+D\sqrt{1/\ntr}+\rk$$

\subsection{Proof of Theorem \ref{t2r2}} 

The result comes directly from theorems \ref{t1r1}, \ref{t1r2} and \ref{t1r3}

\hfill $\blacksquare$

\subsection{Proof of Theorem \ref{btt}}

Given $\spf\in\act{\flt}$ with $|\spf|>\fns$ we have $\cplx{\flt}(\spf)=\ity$. Since, for all $\spf'\in\act{\flo(2N, C, C+D)}$ we have $\cplx{\flo(2N, C, C+D)}(\spf')\leq\ity$ we trivially have that $\cplx{\flo(2N, C, C+D)}(\at{\trt}(\spf))\leq\cplx{\flt}(\spf)$. On the other hand, suppose we have $\spf\in\act{\flt}$ with $|\spf|\leq\fns$. Then $|\at{\trt}(\spf)|=|\spf|+|\{N+i : i\leq \fns - |\spf|\}|=|\spf|+\fns-|\spf|=\fns$ so $\cplx{\flo(2N, C, C+D)}(\at{\trt}(\spf))=0=\cplx{\onflt}(\spf)$. So in any case $\cplx{\flf{\fns}(2N, C, C+D)}(\at{\trt}(\spf))\leq\cplx{\flt}(\spf)$.

Now suppose we have some $f\in\lost{\flt}$. Let $\boc, \bdv$ be such that $f=\fll{\boc}{\bdv}$ and let $\hat{\bocv{}},\hat{\bdv}\in\mathbb{R}^{2N}$ be defined as:
$$\hat{c}_{i}:=c_i,~ ~~\hat{d}_i:=d_i~~~~~~\forall i\in\na{N}$$
$$\hat{c}_i:=0, ~~~\hat{d}_i:=C+D~~~~~~\forall i\in\na{2N}\setminus\na{N}$$
Note then that $\ft{\trt}(\fll{\bocv{}}{\bdv})=\fll{\hat{\bocv{}}}{\hat{\bdv}}$. Since $\hat{\boc}\in[0,C]^{2N}$ and $\hat{\bdv}\in[0,C+D]^{2N}$ we have that $\ft{\trt}(\fll{\bocv{}}{\bdv})\in\lost{\onflo(2N, C, C+D,\fns)}$. All that is left to show is that $[\ft{\trt}(\fll{\bocv{}}{\bdv})](\spf')\geq\int_{\act{\onflt}}\fll{\bocv{}}{\bdv}~[\pt{\trt}(\spf')]$ for all $\spf'\in\act{\flo}$ and $(\bocv{},\bdv)\in[0,C]^N\times[0,D]^N$. We have two cases:
\begin{itemize}
\item In the case that $\spf'\cap\na{N}=\emptyset$ then for all $i\in\spf'$ we have $i\in\na{2N}\setminus\na{N}$ so $\hat{c}_i=0$ and $\hat{d}_i=C+D$. This means that:
\begin{align}
\notag[\ft{\trt}(\fll{\bocv{}}{\bdv})](\spf')&=\sum_{i\in\spf'}\hat{c}_i+\min_{i\in\spf'}\hat{d}_i=C+D\geq c_1+d_1\\
\notag&=\fll{\bocv{}}{\bdv}(\{1\})=\int_{\act{\onflt}}\fll{\bocv{}}{\bdv}~[\dpm{\act{\flb{\fns}}}{\{1\}}]=\int_{\act{\onflt}}\fll{\bocv{}}{\bdv}~[\pt{\trt}(\spf')]
\end{align} 
\item In the case that $\spf'\cap\na{N}\neq\emptyset$ choose $i\in\spf'\cap\na{N}$ that minimises $\hat{d}_i$. Since $i\in\na{N}$ we have $\hat{d}_i\leq D<C+D=\hat{d}_j$ for all $j\in\spf'\setminus\na{N}$ so:
$$\min_{j\in\spf'}\hat{d}_j=\hat{d}_i=\min_{j\in\spf'\cap\na{N}}\hat{d}_j=\min_{j\in\spf'\cap\na{N}}{d}_j$$
also we have:
\begin{align}
\notag\sum_{j\in\spf'}\hat{c}_j&=\sum_{j\in\spf\cap\na{N}}\hat{c}_j+\sum_{j\in\spf\cap(\na{2N}\setminus\na{N}}\hat{c}_j\\
\notag&=\sum_{j\in\spf\cap\na{N}}c_j+\sum_{j\in\spf\cap(\na{2N}\setminus\na{N}}0\\
\notag&=\sum_{j\in\spf\cap\na{N}}c_j
\end{align}
so:
\begin{align}
\notag[\ft{\trt}(\fll{\bocv{}}{\bdv})](\spf')&=\sum_{j\in\spf'}\hat{c}_j+\min_{j\in\spf'}\hat{d}_j\\
\notag&=\sum_{j\in\spf\cap\na{N}}c_j+\min_{j\in\spf'\cap\na{N}}{d}_j\\
\notag&=\fll{\bocv{}}{\bdv}(\spf\cap\na{N})\\
\notag&=\int_{\act{\onflt}}\fll{\bocv{}}{\bdv}~[\dpm{\act{\flb{\fns}}}{\spf'\cap\na{N}}]\\
\notag&=\int_{\act{\onflt}}\fll{\bocv{}}{\bdv}~[\pt{\trt}(\spf')]
\end{align}
\end{itemize}

\hfill $\blacksquare$

\subsection{Proof of Theorem \ref{t2r1}}
Given $\rl, \rk\in\rplus$ suppose we have $(\spf,f)\in\act{\flt}\times\lost{\flt}$ with $f(\spf)\leq\rl$ and $\cplx{\flt}(\spf)\leq\rk$

If $\rk\geq\ity$ we trivially have that $[\ft{\trt}(f)](\at{\trt}(\spf))<\ity\leq\rl+\rk$ so suppose now that $\rk<\ity$. Then we have that $\cplx{\flt}(\spf)<\ity$ and hence $|\spf|\leq\fns$.

Let $\boc, \bdv$ be such that $f=\fll{\boc}{\bdv}$ and let $\hat{\bocv{}},\hat{\bdv}\in\mathbb{R}^{2N}$ be defined as:
$$\hat{c}_{i}:=c_i,~ ~~\hat{d}_i:=d_i~~~~~~\forall i\in\na{N}$$
$$\hat{c}_i:=0, ~~~\hat{d}_i:=C+D~~~~~~\forall i\in\na{2N}\setminus\na{N}$$
Note then that $\ft{\trt}(\fll{\bocv{}}{\bdv})=\fll{\hat{\bocv{}}}{\hat{\bdv}}$ so: 
\begin{align}
\notag[\ft{\trt}(\fll{\bocv{}}{\bdv})](\at{\trt}(\spf))&=\fll{\hat{\bocv{}}}{\hat{\bdv}}(\at{\trt}(\spf))\\
\notag&=\sum_{i\in\at{\trt}(\spf)}\hat{c}_{i}+\min_{i\in\at{\trt}(\spf)}\hat{d}_i\\
\notag&=\sum_{i\in\at{\trt}(\spf)}\hat{c}_{i}+\min_{i\in\at{\trt}(\spf)\cap\na{N}}\hat{d}_i\\
\notag&=\sum_{i\in\at{\trt}(\spf)}\hat{c}_{i}+\min_{i\in\at{\trt}(\spf)\cap\na{N}}{d}_i\\
\notag&=\sum_{i\in\at{\trt}(\spf)\cap\na{N}}c_i+\min_{i\in\at{\trt}(\spf)\cap\na{N}}{d}_i\\
\notag&=\sum_{i\in\spf}c_i+\min_{i\in\spf}{d}_i\\
\notag&=\fll{\bocv{}}{\bdv}(\spf)\\
\notag&=f(\spf)\\
\notag&\leq\rl\\
\notag&\leq\rl+\rk
\end{align}

So, in either case, we have $[\ft{\trt}(\fll{\bocv{}}{\bdv})](\at{\trt}(\spf))\leq\rl+\rk$. Since this applies to all $(\spf,f)$ with $f(\spf)\leq\rl$ and $\cplx{\flt}(\spf)\leq\rk$ we hence have $\lm{\tran}(\rl,\rk)\leq\rl+\rk$

\hfill $\blacksquare$

\subsection{Proof of Theorem \ref{bflth}}

Direct from theorems \ref{t1r1}, \ref{t2r2} and \ref{t2r1}

\hfill $\blacksquare$

\subsection{Proof of Theorem \ref{finth}}

Note first that since $\min\{\cplx{\flg}(\lac)~|~\lac\in\act{\flg}\}=1$ we also have $\min\{\cplx{\g}(\lac)~|~\lac\in\act{\g}\}=1$ which is required to use the doubling trick. We define the quantities $\nosf$ and $\nosg$ as equal to $\lceil\ln(\ntr)/2\rceil$ and $(a+b)\sqrt{\ln(2N)/\ntr}$ respectively. Defining $\g^{\trsh}$ (from $\g$), for all $\trsh\geq 1$, as in Subsection \ref{GDTss} we have, for all $\spf\in\act{\g}$:
\begin{align}
\notag\cplx{\g^{\trsh}}(\spf)&=\cplx{\g}^{\trsh}(\spf)\\
\notag&=\indi{\cplx{\g}(\spf)>\trsh}\ity\\
\notag&=\indi{(a\cplx{\flg}(\spf)+b)/(a+b)>\trsh}\ity\\
\notag&=\indi{\cplx{\flg}(\spf)>((a+b)\trsh-b)/a}\ity\\
\notag&=\indi{|\spf|>((a+b)\trsh-b)/a}\ity\\
\notag&=\indi{|\spf|>\lfloor(a+b)\trsh-b)/a\rfloor}\ity\\
\notag&=\cplx{\flt(N, C, D, \lfloor(\trsh(a+b)-b)/a\rfloor)}(\spf)
\end{align}
Hence we have that $\g^{\trsh}=\flt(N, C, D, \lfloor(\trsh(a+b)-b)/a\rfloor)$
so by Theorem \ref{bflth} we have that the strategy $\bstrat^{\trsh}$ has a generalised regret, with respect to $\g^{\trsh}$, of 
\begin{align}
\notag&\gnrs{\g^{\trsh}}{\bstrat^{\trsh}}(\rl,\rk)\\
\notag=&\gnrs{\flt(N, C, D, \lfloor(\trsh(a+b)-b)/a\rfloor)}{\bstrat^{\trsh}}(\rl,\rk)\\
\notag=&\gnrs{\flt(N, C, D, \lfloor(\trsh(a+b)-b)/a\rfloor)}{\bstrat^{\operatorname{\onflt(N, C, D, \lfloor(\trsh(a+b)-b)/a\rfloor)}}}(\rl,\rk)\\
\notag\leq&\lceil\ln(\ntr)\rceil\rl/2+(2\lfloor\trsh(a+b)-b)/a\rfloor\lceil\ln(\ntr)/2\rceil(2C+D)+(C+D))\sqrt{\ln(2N)/\ntr}+\rk\\
\notag\leq&\lceil\ln(\ntr)\rceil\rl/2+(2(\trsh(a+b)-b)/a)\lceil\ln(\ntr)/2\rceil(2C+D)+(C+D))\sqrt{\ln(2N)/\ntr}+\rk\\
\notag=&\lceil\ln(\ntr)\rceil\rl/2+((\trsh(a+b)-b)/a)a +b)\sqrt{\ln(2N)/\ntr}+\rk\\
\notag=&\lceil\ln(\ntr)\rceil\rl/2+\trsh(a+b)\sqrt{\ln(2N)/\ntr}+\rk\\
\notag=&\nosf\rl+\nosg\trsh+\rk
\end{align}
which is required for the doubling trick. Since all the conditions for the doubling trick are now satisfied we can invoke Theorem \ref{dtth}, giving us:
\begin{equation}\label{dtsubeq}
\gnrs{\g}{\bstrat^{\dt}}(\rl,\rk)\leq 5\nosf\rl+8\nosg\rk+\frac{1}{\ntr}\sum_{i=1}^{\lceil\log_2(\rk)\rceil}\maxl{2^i}
\end{equation}
where $\maxl{2^i}$ is defined as in Subsection \ref{GDTss}. That is:
$$\maxl{\trsh}:=\max\left\{\ints_{\act{}}\nac_{t}[\strat{t}^{\trsh}(\bnac)]~|~\bnac\in\lost{}^{\ntri{}},t\in\na{\ntri{}}\right\}$$
which, by above, is equal to:
$$\max\left\{\ints_{\act{}}\nac_{t}[\bstrat^{\operatorname{\onflt(N, C, D, \lfloor(\trsh(a+b)-b)/a\rfloor)}}(\bnac)]~|~\bnac\in\lost{}^{\ntri{}},t\in\na{\ntri{}}\right\}$$
Fix some $\rk\in\rplus$ and let $\fns:=(\rk(a+b)-b)/a$.

Since the strategy $\bstrat^{\operatorname{\onflt(N, C, D, \lfloor\trsh(a+b)-b)/a\rfloor)}}$ always selects a set of at most $\lceil\ln(\ntr)\rceil\lfloor(\trsh(a+b)-b)/a\rfloor$ sites we have $\maxl{\trsh}\leq \lceil\ln(\ntr)/2\rceil\lfloor(\trsh(a+b)-b)/a\rfloor C+D$
so we have:
\begin{align}
\notag\sum_{i=1}^{\lceil\log_2(\rk)\rceil}\maxl{2^i}&\leq\sum_{i=1}^{\lceil\log_2(\rk)\rceil}\lceil\ln(\ntr)/2\rceil\left\lfloor\frac{2^i(a+b)-b}{a}\right\rfloor C+D\\
\notag&\leq\sum_{i=1}^{\lceil\log_2(\rk)\rceil}\lceil\ln(\ntr)/2\rceil\frac{2^i(a+b)-b}{a}C+D\\
\notag&\leq\sum_{i=1}^{\lceil\log_2(\rk)\rceil}\lceil\ln(\ntr)/2\rceil\frac{2^i(a+b)}{a}C+D\\
\notag&\leq\lceil\ln(\ntr)/2\rceil\frac{2^{\lceil\log_2(\rk)\rceil+1}(a+b)}{a}C+D\\
\notag&\leq\lceil\ln(\ntr)/2\rceil\frac{8\rk(a+b)}{a}C+D\\
\notag&=\lceil\ln(\ntr)/2\rceil8\frac{\rk(a+b)-b}{a}C+8\lceil\ln(\ntr)/2\rceil\frac{b}{a}C+D\\
\notag&=\lceil\ln(\ntr)/2\rceil8\frac{\rk(a+b)-b}{a}C+8\lceil\ln(\ntr)/2\rceil\frac{C+D}{(4C+2D)\lceil\ln(\ntr)/2\rceil}C+D\\
\notag&=\lceil\ln(\ntr)/2\rceil8\frac{\rk(a+b)-b}{a}C+8\frac{C+D}{4C+2D}C+D\\
\notag&=\lceil\ln(\ntr)/2\rceil8\frac{\rk(a+b)-b}{a}C+8C+D\\
\notag&=4\lceil\ln(\ntr)/2\rceil\fns C+8C+D
\end{align}
so:
\begin{align}
\notag\frac{1}{\ntr}\sum_{i=1}^{\lceil\log_2(\rk)\rceil}\maxl{2^i}&\in\mathcal{O}\left(\frac{\ln(\ntr)}{\ntr}(\fns C+C+D)\right)\\
\label{dtsub1}&\subseteq\mathcal{O}\left(\fns(\mcos+\mdis)\sqrt{\frac{\ln(\nsi)}{T}} \right)
\end{align}
We also have:
\begin{align}
\notag\nosg\rk&=\rk (a+b)\sqrt{\frac{\ln(2N)}{\ntr}}\\
\notag&\frac{a\fns+b}{a+b}(a+b)\sqrt{\frac{\ln(2N)}{\ntr}}\\
\notag&=(a\fns+b)\sqrt{\frac{\ln(2N)}{\ntr}}\\
\label{dtsub2}&\in\mathcal{O}\left(\fns(C+D)\ln(\ntr)\sqrt{\frac{\ln(N)}{\ntr}}\right)
\end{align}
Combining equations \eqref{dtsubeq}, \eqref{dtsub1} and \eqref{dtsub2} gives us:
$$\gnrs{\g}{\bstrat^{\dt}}(\rl,\rk)\in\mathcal{O}\left(\rl\ln(\ntr) +\fns(\mcos+\mdis)\ln(T)\sqrt{\frac{\ln(\nsi)}{T}} \right)$$
which implies the result.

\hfill $\blacksquare$

\section{Hypothesis Classes and Infinite Complexities}\label{infclsec}
In this paper we utilise complexity functions that (informally) evaluate as infinite on some actions. Hence, to give an idea of what these infinities mean, we now consider, as an example, the general task of ``online classification''. Since infinity is not actually a number we will instead use, as a surrogate, a number $\ity$ and take the limit $\ity\rightarrow\infty$.

 In an online classification problem we have a set $\gens$ and a set $\mathcal{H}$ of functions from $\gens$ into $\{-1,\ual,1\}$ that are known to Learner. We call $\mathcal{H}$ the ``hypothesis space''. We also have a ``complexity'' function $\cplx{}':\mathcal{H}\rightarrow\rplus$. Nature chooses some $h\in\mathcal{H}$ a-priori but doesn't reveal it to Learner. Learning proceeds in trials $t=1,2,...,\ntr$. On trial $t$:
\begin{enumerate}
\item Nature chooses some $s_t\in\gens$ with $h(s_t)\neq\ual$ and reveals it to Learner.
\item Learner chooses some $\hat{y}_t\in\{-1,1\}$
\item $h(s_t)$ is revealed to Learner
\item If $\hat{y}_t\neq h(s_t)$ then Learner incurs a mistake.
\end{enumerate}
Given a strategy for Learner we define its ``mistake bound'' to be a function $M:\rplus\rightarrow\rplus$ such that $M(\beta)$ is the maximum number of mistakes made by the algorithm if nature chooses $h$ with $\cplx{}'(h)\leq\beta$. 

An example of online classification is ``online linear classification'' of dimension $n$ in which $\gens=\{\bs{s}\in\mathbb{R}^n:\|\bs{s}\|\leq1\}$ and each hypothesis $h$ is defined by a pair $(\bs{w},\mu)\in\gens\times(0,1)$ such that $$h(\bs{s})=\indi{|\bs{w}\cdot\bs{s}|\geq\mu}\operatorname{sign}(\bs{w}\cdot\bs{s})~~~\operatorname{and}~~~\cplx{}'(h)=1/\mu$$
The famous ``Perceptron'' algorithm achieves a mistake bound of $M(\beta)\leq\beta^2$ for online linear classification.

We can formulate online classification as an online optimisation game $\g$ as follows:
\begin{itemize}
\item $\act{\g}=\sfun{\gens}{\{-1,\ual,1\}}$
\item $\lost{\g}$ is the set of all $\nac\in\sfun{\act{\g}}{\rplus}$ such that there exists $(s,y)\in\gens\times\{-1,1\}$ with $\nac(g)=\indi{g(s)\neq y}$  for all  $g\in\act{\g}$
\item Given $g\in\act{\g}$ we have $\cplx{\g}(g):=\indi{g\notin\mathcal{H}}\ity+\indi{g\in\mathcal{H}}\cplx{}'(g)$
\end{itemize}
This game is equivalent to online classification by the following relationships (where $\nac_t$ and $g_t$ are Nature's and Learner's actions on trial $t$ respectively).
\begin{itemize}
\item $\nac_t(g_t)=\indi{g_t(s_t)\neq h(s_t)}$
\item $\hat{y}_t=g_t(s_t)$
\item If a mistake is made on trial $t$ then $\nac_t(g_t)=1$. Otherwise $\nac_t(g_t)=0$
\end{itemize}

Given a strategy $\bstrat$ with mistake bound $M(\cdot)$, its generalised regret $\gnrs{\g}{\bstrat}$ satisfies:
$$\gnrs{\g}{\bstrat}(\rl,\rk)\leq \rl\ity+\frac{1}{\ntr}M(\rk)+\ity\indi{\rk\geq\ity}$$

So, in Online classification, complexities evaluate as infinite on actions that correspond to functions that are not in the hypothesis space: i.e. those that Nature cannot choose.

\section{The Failure of Deterministic Algorithms}\label{failsec}

In this section we prove that no deterministic algorithm, e.g. follow the (approximate) leader, can achieve the (expected) loss bound of our algorithm; even if the opening costs do not vary from trial to trial. Specifically we prove the following theorem and corollary:

\begin{theorem}\label{failth}
Take the online learning problem of Section \ref{PDAR} with $C:=1$ and $D:=1$.  Suppose we have a deterministic algorithm $\ona$ for Learner and a function $B:(\rplus)^3\rightarrow \rplus$ such that $B$ is monotonic increasing in its first variable and, for all $\epsilon\in\rplus$\,, there exists $N,T\in\mathbb{N}$ with $B(2/\sqrt{N}, N, T)<\epsilon$. Then, for all $\epsilon\in\rplus$ there exists an $N$ and $T$ such that there are $N$ sites, $T$ trials, and a sequence of Natures selections of cost vectors $\{{\bocv{t}},{\bccm{t}}\in[0,1]^N:t\in\na{\ntr}\}$, and a set $\sels{*}\in\pows{\na{N}}\setminus\emptyset$ such that:
$$B\left(\frac{1}{\ntr}\sum_{t=1}^T\fll{\bocv{t}}{\bccm{t}}(\sels{*}), N, T\right)<\epsilon$$
and:
$$\frac{1}{\ntr}\sum_{t=1}^T\fll{\bocv{t}}{\bccm{t}}(\sels{t})\geq 1$$
where $\sels{t}$ is selection of $\ona$ at trial $t$. In addition, the selection of opening cost vectors need not vary from trial to trial. i.e. there exist $\bocv{}\in[0,1]^N$ such that $\bocv{t}:=\bocv{}$ for all $t\in\na{\ntr}$.
\end{theorem}

Theorem \ref{failth} has the following corollary:

\begin{corollary}\label{failco}
For any deterministic algorithm $\ona$ for Learner, for the problem of Section \ref{PDAR} then:
$$\mathbb{E}\left(\sum_{t=1}^{\ntr}\fll{\bocv{t}}{\bccm{t}}(\sels{t})\right)\notin\mathcal{O}\left(\ln(\ntr)\sum_{t=1}^{\ntr}\fll{\bocv{t}}{\bccm{t}}(\sels{*})+N(\mcos+\mdis)\ln(\ntr)\sqrt{{\ln(\nsi)}{\ntr}}\right)$$
where $\sels{t}$ is selection of $\ona$ on trial $t$ and $\sels{*}$ is an arbitrary set of sites.
\end{corollary}

Corollary \ref{failco} follows from Theorem \ref{failth} by choosing the function $B$ (in Theorem \ref{failth}) such that
$B(L,N,T):={\ln(T)}L+2N\ln(\ntr)\sqrt{{\ln(N)}/{T}}$

We now prove Theorem \ref{failth}. Suppose we have some arbitrary $\epsilon\in\rplus$. Choose $N$ and $T$ such that $B(2/\sqrt{N}, N, T)<\epsilon$. Let $\sels{t}$ be the selection of $\ona$ at trial $t$. We define the fixed opening cost vector $\bfc$ by $\fc_i:=1/\srn$ for all $i\in[N]$.

\begin{definition}
We partition $[T]$ into two sets, $\Lambda$ and $\Upsilon$, where:
\begin{itemize}
\item $\Lambda:=\{t\in[T]:|\spl{t}|\leq\srn\}$
\item $\Upsilon:=\{t\in[T]:|\spl{t}|>\srn\}$
\end{itemize}
\end{definition}

Since the algorithm is deterministic, Nature can know the choice of $\spl{t}$ before the it chooses $\bdis{t}$. Hence, we now define an adversarial choice of this vector:
\begin{itemize}
\item If $t\in\Lambda$ then for all $i\in\spl{t}$ set $\dis{t}{i}:= 1$ and for all $i\in[N]\setminus\spl{t}$ set $\dis{t}{i}:= 0$.
\item If $t\in\Upsilon$ then for all $i\in[N]$ set $\dis{t}{i}:=0$
\end{itemize}

\begin{lemma}\label{detl1}
We have:
$$\frac{1}{T}\sum_{t\in[T]}\fll{\bocv{t}}{\bccm{t}}(\spl{t})\geq1$$
\end{lemma}

\begin{proof}
Note that if $t\in\Lambda$ we have:
\begin{align}
\notag\fll{\bocv{t}}{\bccm{t}}(\spl{t})&\geq\min_{i\in[N]}\dis{t}{i}\\
\notag&=\min_{i\in[N]}1\\
\notag&=1
\end{align}
and if $t\in\Gamma$ we have:
\begin{align}
\notag\fll{\bocv{t}}{\bccm{t}}(\spl{t})&\geq\sum_{i\in[N]}\cst{t}{i}\\
\notag&=\sum_{i\in\spl{t}}\fc_i\\
\notag&=\sum_{i\in\spl{t}}1/\srn\\
\notag&=|\spl{t}|/\srn\\
\notag&\geq\srn/\srn\\
\notag&=1
\end{align}
So in either case we have $\fll{\bocv{t}}{\bccm{t}}(\spl{t})\geq 1$ and hence $\frac{1}{T}\sum_{t\in[T]}\fll{\bocv{t}}{\bccm{t}}(\spl{t})\geq 1$.
\end{proof}

\begin{lemma}\label{dett1}
There exists $\sels{*}\in\pows{[N]}\setminus\emptyset$ such that:
$$B\left(\frac{1}{\ntr}\sum_{t=1}^T\fll{\bocv{t}}{\bccm{t}}(\sels{*}), N, T\right)<\epsilon$$
\end{lemma}

\begin{proof}
We define $j:=\operatorname{argmin}_{i\in[N]}|\{t\in\Lambda:i\in\spl{t}\}|$ and define $\bset:=\{j\}$. We have:
\begin{align}
\notag\sum_{i\in[N]}|\{t\in\Lambda:i\in\spl{t}\}|&=\sum_{i\in[N]}\sum_{t\in\Lambda}\ind{i\in\spl{t}}\\
\notag&=\sum_{t\in\Lambda}\sum_{i\in[N]}\mathcal{I}(i\in\spl{t})\\
\notag&=\sum_{t\in\Lambda}|\spl{t}|\\
\notag&\leq\sum_{t\in\Lambda}\srn\\
\notag&\leq T\srn
\end{align}
 Hence we have that:
\begin{align}
\notag|\{t\in\Lambda:j\in\spl{t}\}|&\leq(1/N)\sum_{i\in[N]}|\{t\in\Lambda:i\in\spl{t}\}|\\
\notag&\leq (1/N)T\srn\\
\notag&=T/\srn
\end{align}
Note that on trial $t\in\Lambda$ we have $\dis{t}{j}=1$ if $j\in\spl{t}$ and $\dis{t}{j}=0$ otherwise. Also on a trial $t\in\Gamma$ we have  $\dis{t}{j}=0$. This means:
\begin{align}
\notag\sum_{t\in[T]}\dis{t}{j}&=\sum_{t\in\Lambda}\dis{t}{j}+\sum_{t\in\Gamma}\dis{t}{j}\\
\notag&=\sum_{t\in\Lambda}\dis{t}{j}\\
\notag&=\sum_{t\in\Lambda:j\in\spl{t}}\dis{t}{j}+\sum_{t\in\Lambda:j\notin\spl{t}}\dis{t}{j}\\
\notag&=|\{t\in\Lambda: j\in\spl{t}\}|\\
\notag&\leq T/\srn
\end{align}
And hence, since $\bset=\{j\}$\,:
\begin{align}
\notag\frac{1}{T}\sum_{t\in[T]}\lf{t}(\bset)&=\frac{1}{T}\sum_{t\in[T]}\left(\cst{t}{j}+\dis{t}{j}\right)\\
\notag&=\frac{1}{\srn}+\frac{1}{T}\sum_{t\in[T]}\cst{t}{j}\\
\notag&=\frac{1}{\srn}+\frac{1}{T}\sum_{t\in[T]}\frac{1}{\srn}\\
\notag&=2/\srn
\end{align}
So $B$ is monotonic increasing in its first variable we then have:
$$B\left(\frac{1}{\ntr}\sum_{t=1}^T\fll{\bocv{t}}{\bccm{t}}(\sels{*}), N, T\right)<B\left(2/\sqrt{N}, N, T\right)$$
which, by our choice of $N$ and $T$, is bounded above by $\epsilon$.
\end{proof}

Lemmas \ref{detl1} and \ref{dett1} imply Theorem \ref{failth}.

\hfill $\blacksquare$

\end{appendix}

\end{document}